\documentclass{aims} 
\usepackage{amsmath,amssymb,amsthm}
\usepackage{bm}
\usepackage{makecell} 
\usepackage{multirow}
\usepackage{algpseudocode}
\usepackage{paralist}
\usepackage{flafter}
\usepackage{graphicx,graphics}
\usepackage[misc]{ifsym}
\usepackage{epsfig} 
\usepackage{epstopdf} 
\usepackage[colorlinks=true]{hyperref}
\hypersetup{urlcolor=blue, citecolor=red}

\usepackage{amsmath}
\usepackage{amssymb}
\usepackage{amsthm}
\usepackage{amsfonts}
\usepackage{mathrsfs}
\usepackage{graphicx}
\usepackage{cite}

\usepackage{ulem}
\usepackage{soul}
\usepackage{fixltx2e}
\usepackage{float}
\usepackage{booktabs}
\usepackage[table]{xcolor}
\usepackage[ruled, linesnumbered]{algorithm2e}
\usepackage{array}
\usepackage{caption}
\usepackage{subcaption}
\usepackage{multicol}
\usepackage{multirow} 
\usepackage{diagbox}
\usepackage{makecell}
\usepackage{algorithm2e}
\usepackage{xcolor}

\definecolor{mypink1}{RGB}{231, 136, 149}
\definecolor{green1}{RGB}{158, 200, 185}

\allowdisplaybreaks

\def\currentvolume{X}
 \def\currentissue{X}
  \def\currentyear{200X}
   \def\currentmonth{XX}
    \def\ppages{X-XX}
     \def\DOI{10.3934/xx.xxxxxxx}

\newtheorem{theorem}{Theorem}[section]
\newtheorem{corollary}[theorem]{Corollary}

\newtheorem{proposition}[theorem]{Proposition}

\theoremstyle{definition}
\newtheorem{definition}[theorem]{Definition}
\newtheorem{remark}[theorem]{Remark}

\title[Neural Shr\"odinger Bridge Matching for Pansharpening]
{Neural Shr\"odinger Bridge Matching for Pansharpening} 

\author[Zihan Cao, Xiao Wu, Liang-jian Deng]{}


\thanks{$^*$Corresponding author: Liang-Jian Deng.}

\begin{document}
\maketitle

\centerline{\scshape
Zihan Cao$^{{\href{iamzihan666@gmail.com}{\textrm{\Letter}}}1}$, Xiao Wu$^{{\href{tingzhuhuang@126.com}{\textrm{\Letter}}}1}$, Liang-Jian Deng$^{{\href{liangjian.deng@uestc.edu.cn}{\textrm{\Letter}}}*1}$}

\medskip

{\footnotesize
 \centerline{$^1$School of Mathematical Sciences, University of Electronic Science and Technology of China,}
 \centerline{Chengdu, 611731, China}
} 





\bigskip


\begin{abstract}
	Recent diffusion probabilistic models (DPM) in the field of pansharpening have been gradually gaining attention and have achieved state-of-the-art (SOTA) performance. In this paper, we identify shortcomings in directly applying DPMs to the task of pansharpening as an inverse problem: 1) initiating sampling directly from Gaussian noise neglects the low-resolution multispectral image (LRMS) as a prior; 2) low sampling efficiency often necessitates a higher number of sampling steps. We first reformulate pansharpening into the stochastic differential equation (SDE) form of an inverse problem. Building upon this, we propose a Schr\"odinger bridge matching method that addresses both issues.
	We design an efficient deep neural network architecture tailored for the proposed SB matching.
	In comparison to the well-established DL-regressive-based framework and the recent DPM framework, our method demonstrates SOTA performance with fewer sampling steps. Moreover, we discuss the relationship between SB matching and other methods based on SDEs and ordinary differential equations (ODEs), as well as its connection with optimal transport.
	Code will be available.
\end{abstract}

\section{Introduction}
Pansharpening is a special case of super-resolution and falls into the category of classical inverse image restoration problems in remote sensing.
Due to the inherent limitations of the hardware system, existing imaging devices such as Gaofen-2 (GF2), WorldView-4 (WV4), and WorldView-3 (WV3) can only measure low-resolution multispectral (LRMS) $\in \mathbb R^{h\times w\times C}$ images while capturing finer spatial information into grayscale panchromatic (PAN) $\in \mathbb R^{H\times W\times c}$ images (where $c<C$ and $h<H, w<W$). In general, such a trade-off imposes restrictions on many downstream applications that rely on high-resolution multispectral (HRMS) $\in \mathbb R^{H\times W\times C}$ images.

Various approaches have been proposed in the past to address the pansharpening problem. These include model-based methods such as component substitution (CS)~\cite{cs1,cs2,cs3},  multi-resolution analysis (MRA)~\cite{mra1,mra2,mra3,mra4}, and variational optimization (VO) methods~\cite{vo1,vo2,vo3,vo4}, as well as deep regression models~\cite{deng2020detail,deng2023psrt,lagconv,hmpnet,pnn,msdcnn,dcfnet} and the recent diffusion models~\cite{pandiff,DDIF,wu2023hsr}.
Previous model-based methods typically involve transforming pansharpening into an optimization problem using physical models, often yielding suboptimal results. In contrast, deep regression models leverage deep learning by feeding LRMS and PAN inputs into a carefully designed deep neural network to fuse and generate HRMS. Much of the prior work has focused on contributing to the design of more efficient deep neural networks. There are also some works that integrate physical models into the design of their models. They incorporate physical formulas into the forward process of the network or design the solution of an optimization problem as a module within the network. While these methods often achieve good performance, they may introduce additional computational overhead.

The recent diffusion-based model (DPM) has garnered significant attention in various fields such as image and video generation~\cite{liu2024sora,saharia2022image,ho2022imagen,saharia2022palette}, molecular structure generation~\cite{xu2022geodiff}, and reinforcement learning~\cite{zhu2023diffusion}. Many works have applied DPM to pansharpening, achieving state-of-the-art (SOTA) results. The approach of DPM relies on stochastic differential equations (SDEs) or ordinary differential equations (ODEs). Its generation process starts from an easily samplable distribution (\textit{e.g.}, Gaussian distribution) and progressively denoises the Gaussian distribution into the target distribution.
Due to the nature of being an SDE/ODE model, DPM requires multiple steps of network forward evaluation (NFE) during sampling, resulting in a lengthy sampling process. This limitation hinders the current development of DPM. Despite recent efforts to develop efficient SDE or ODE DPM samplers, the quality of sampled outputs remains unsatisfactory when NFE is low. Furthermore, some works utilize distillation to transfer knowledge from a large NFE network to a smaller NFE one, but this still requires additional data and computational resources for distillation.
For solving inverse problems, DPM also finds many applications. Some methods based on singular value decomposition (SVD)~\cite{ddrm} and optimization~\cite{chung2023diffusion} modify the sampling process of DPM, integrating the solution process of the inverse problem into it, yielding promising results. However, these methods require DPM trained on corresponding data; otherwise, due to the significant domain gap (\textit{e.g.,} using Imagenet trained DPM model to restore the remote sensing images), the quality of sampling may degrade.

Recently, DDIF~\cite{DDIF} and PanDiff~\cite{pandiff} inheriting the concept of DPM, utilizes LRMS and PAN as conditional inputs into a diffusion network based on the VP SDE~\cite{ho2020denoising}, performing denoising to transform the Gaussian samples into HRMS. 
Although they have achieved satisfactory pansharpening performance, they all overlook a crucial fact: \textit{pansharpening, as a type of inverse problem, benefits from the known degradation of LRMS to HRMS, which can serve as a prior, they still initiate the sampling process from a Gaussian distribution}, which is unreasonable and counterintuitive.

The Schr\"odinger Bridge (SB) problem was first introduced in quantum mechanics~\cite{schrodinger} and expanded to broader fields such as OT problem. It seeks two optimal policies that transform back-and-forth between two arbitrary distributions in a ﬁnite horizon.
In recent times, there has been a surge in efforts to develop more efficient neural SB solvers for solving SB problems. Examples include solvers based on IPF~\cite{de2021diffusion,shi2024diffusion} and likelihood-based approaches~\cite{neklyudov2022action,chen2021likelihood,wang2021deep}. However, these methods involve \textit{trajectory simulation (simulating the entire SDE trajectory), training multiple forward-backward networks simultaneously, complex learning objectives, and intricate training procedures}. Consequently, applying these methods to large-scale image tasks is not practical.

To address the above issues of 1) long sampling time, 2) starting sampling from a Gaussian distribution in pansharpening of current DPM-based methods, and 3) the impracticality of neural SB solvers for large-scale image problems, we first characterize the pansharpening task based on the characteristics of traditional CS and MRA methods and deduce the form of the SDE in Sect.~\ref{sec:inverse-pan}. We then extend this to degradation SDE and its ODE form in Sect.~\ref{sec:deg-ode-sde}. Building on this, we simplify the degradation SDE/ODE to linear SDE forms using the SB formulation in Sect.~\ref{sect:simu-free-SB}, thereby reducing the sampling difficulty. Furthermore, we provide the parameterization of the proposed SDE and ODE in Sect.~\ref{sect: bmp-train-sampling-algos.}. The main conception of this paper is illustrated in Fig.~\ref{fig:main_conception}. Extensive experiments are conducted to verify the proposed SB SDE and ODE in Sect.~\ref{sect:exp}. Moreover, further discussion with other works is provided in Sect.~\ref{sect:disscuss}.

To summarize, the contributions of this paper are:
\begin{enumerate} 
	\item We summarize the shortcomings of existing DPM-based pansharpening frameworks. Based on the inverse problem formulation of pansharpening, we express it as a unified degradation-recovery SDE/ODE.
	\item Furthermore, by leveraging the SB formula, we improve the degradation-recovery SDE/ODE to linear forward-backward SDE/ODE. This formulation features an analytic forward process, eliminating the need for simulation and exhibiting excellent properties of optimal transport (OT). This makes both training and sampling highly convenient.
	\item We design a more efficient deep neural network architecture for the proposed SB matching, which enhances the effectiveness of SB SDE.
	\item The proposed SB matching method achieves new SOTA performances on three commonly used pansharpening datasets.
\end{enumerate}


\section{Priliminary}
In this section, we will provide a brief review of the Diffusion Probabilistic Model (DPM)~\cite{ho2020denoising,saharia2022image,song2019generative,song2020score}, Optimal Transport~\cite{villani2009optimal,leonard2013survey,korotin2023neural} that can be linked with Schr\"odinger Bridge (SB)~\cite{chen2021likelihood,neklyudov2022action,de2021diffusion,shi2024diffusion,wang2021deep}.

\subsection{Diffusion Probabilistic Model}
\begin{figure}[t]
	\centering
    \includegraphics[width=0.8\linewidth]{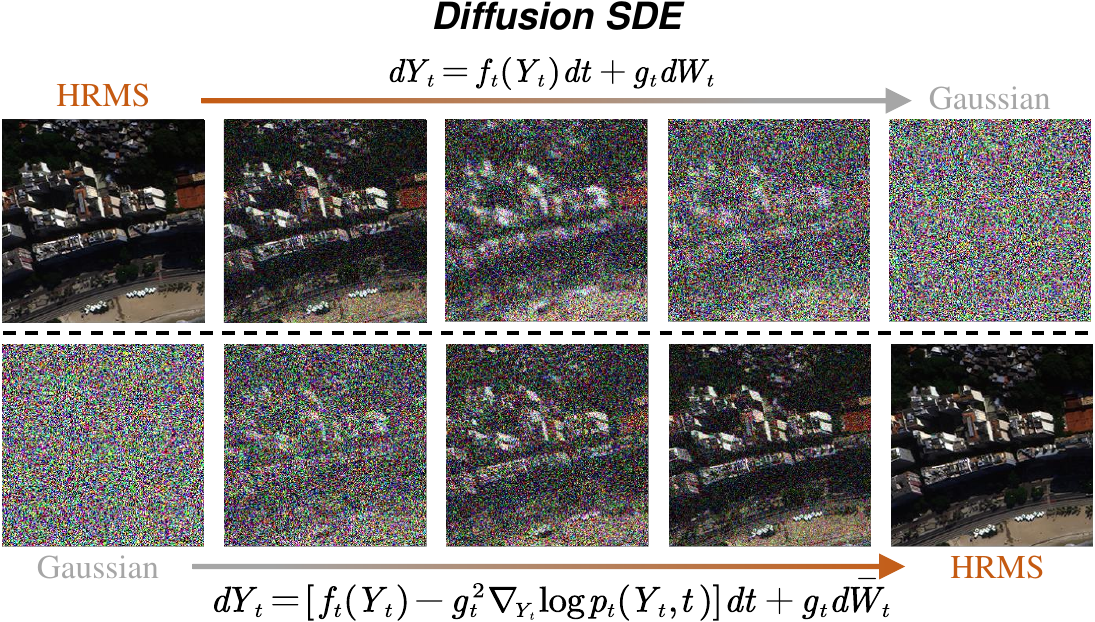}
        \caption{Illustration of diffusion framework. It connects the Gaussian distribution with the HRMS distribution, which is inefficient in handling the pansharpening task. Previous works~\cite{pandiff,DDIF} adopt this scheme.}
	\label{fig:diff-scheme}
	
\end{figure}
Diffusion Probabilistic Model (DPM) uses a little isotropy Gaussian noise to degrade the input image step by step until pure Gaussian noise in the forward process. In the backward process, the trained diffusion model tries to remove the noise from the noisy latent images. DPM forward/backward processes are operated by ODEs or SDEs~\cite{ho2020denoising,song2019generative,karras2022elucidating}. From the viewpoint of score matching~\cite{song2020score,song2019generative} (e.g., a kind of diffusion model), the forward SDE can be formed as,
\begin{equation}
	dX_t=f_t(X_t)dt+g_t dW_t,\ X_0\sim p_X,\ X_1\sim \mathcal N(\mathbf 0, \mathbf I_d), \label{eq:sde}
\end{equation}
where $f_t$ is the drift term and the $g_t$ is the diffusion term, and $W_t$ is the Winner process. We can derive the backward process by using the Fokker Plank (FP) equation~\cite{frank2005nonlinear},
\begin{equation}
	\frac{\partial}{\partial t}p_t(X_t)=-\nabla \cdot (f_t(X_t) p_t(X_t))+\frac{1}{2} g_t^2 \Delta p_t(X_t), \label{eq:fpe}
\end{equation}
where $\nabla \cdot(\cdot)$ is the divergence operator and $\Delta(\cdot)$ is the Laplacian operator. Thus, the backward process can be accomplished,
\begin{equation}
	dX_t=[f_t(X_t)-g_t^2\nabla_{X_t} \log p_t(X_t,t)]dt+g_t d\bar W_t.
\end{equation}
$\bar w_t$ is the reversed Winner process and $\nabla_{X_t} \log p_t$ is known as the score function. 

In practice, it is feasible to analytically sample $X_t$ given $t$ and $X_0$ and to train a neural network $s_\theta$ (often U-Net~\cite{unet}) to predict the score function $\mathbb E_{X_t,t}\|s_\theta(X_t,t;\theta)-\nabla_{X_t}\log p(X_t,t|X_0)\|_2^2$ as same as minimizing the variational lower bound (ELBO),
\begin{equation}
	\begin{aligned}
		L_{vlb}=-\log p_\theta(X_0|X_1) + D_{KL}(p_{T}(X_T|X_0)||p_T(X_T))\\
		+\sum_{t>1}D_{KL}(p_{t-1}(x_{t-1}|x_t,x_0)||p_\theta(X_{t-1}|X_t)).
	\end{aligned}
\end{equation}

Meanwhile, FPE provides an ODE,
\begin{equation}
	dX_t=\left[f_t(X_t)-\frac{1}{2} g_t^2 \nabla_{X_t}\log p_t(X_t,t)\right]dt, \label{eq:diffusion-ode}
\end{equation}
that shares the same marginal distribution which can be effectively solved by some well-studied ODE solvers. Additive parameterization of $f_t(X_t)$ and $g_t$ can be found in the supplementary.

\subsection{Optimal Transport}
The Optimal Transport problem~\cite{villani2009optimal} tries to find the minimal displacement cost from one distribution $p_0$ to another distribution $p_1$ when given a ground cost. Takes a usual two-Wasserstein distance as an example, 
\begin{equation}
	\mathcal W_2({p_0, p_1})^2=\inf_{\pi\in \Pi}\int_{\mathcal X\times \mathcal Y} c(x,y)^2\pi(dx, dy),
\end{equation}
where $\Pi$ is the set of all joint probability measures on space $\mathcal X\times \mathcal Y$ marginalized on $p_X$ and $p_Y$ and $c(\cdot, \cdot)$ denotes the non-negative ground cost. $\pi$ is the OT plan where discrete and continuous OT methods are tried to search. Based on it, we can define the kinetic form given by
\begin{equation}
	\mathcal W_2(q_0,q_1)^2=\inf_{p_t,f_t} \int_{\mathbb R^d}\int_0^1 p_t(X_t)\|f_t\|^2dt dX,
\end{equation}
where $p_t$ is the marginal distribution that obeys the transport equation $\partial_t p_t+\nabla\cdot (p_t f_t)=0$ and has the boundary $p_X,\ p_Y$. The proposed Schr\"odinger Bridge Matching in Sect.~\ref{sect:simu-free-SB} can be linked with Optimal Transport and is discussed in Sect.~\ref{sect:sb-sde-with-ot-relation}.

\subsection{Schr\"odinger Bridge}
SB problem is often defined as
\begin{equation}
	\min_{\mathbb Q\in \mathcal F(p_X, p_Y)}  D_{\text{KL}}(\mathbb Q\| \mathbb P),
	\label{eq: sb-framework}
\end{equation}
where $ \mathcal F(p_X, p_Y) \subset \mathscr P(\Omega)$ is the path measure with the probability density $p_X$ and $p_Y$ at the bridge endpoints and usually $p_X$ is defined as data distribution and $p_Y$ is prior distribution. $\mathbb Q$ and $\mathbb P$ are path measures. 
The goal of an SB problem is to find a $\mathbb P$ that has the same marginal distribution $p_X$ and $p_Y$ at times $t=0$ and $t=1$.
It is also linked with the stochastic optimal control~\cite{dai1991stochastic} which is formulated as follows,
\begin{equation}
	\begin{aligned}
		\min_{u(X_t,t)}&\mathbb E_{p_t}\left[\int_0^1 \frac 1 2 \|u(X_t,t)\|dt\right],\\
		s.t.,\ dX_t&=\left[f_t(X_t)+u_t(X_t,t)\right]dt+g_tdW_t.
		\label{eq:sb-soc}
	\end{aligned}
\end{equation}
This control problem can be converted into a PDE by applying the Hopf-Cole transform~\cite{Hopf1950ThePD},
\begin{subequations}
	\begin{align}
		&\begin{cases}
			\frac{\partial \Psi(X_t,t)}{\partial t} = -\nabla \Psi^\top f_t-\frac 12 \mathrm{Tr}(g_t^2 \nabla_{X_t}^2 \Psi)\\
			\frac{\partial \widehat \Psi(X_t,t)}{\partial t} = -\nabla\cdot (\widehat \Psi^\top f_t)+\frac 12 \mathrm{Tr}(g_t^2 \nabla_{X_t}^2 \widehat \Psi)
		\end{cases} \label{eq:sb-pde}\\
		&\quad s.t,\ \Psi(X_t,0)\widehat\Psi(X_t,0)=p_{\mathcal A}(X)\label{eq:sb-pde-st-a}, \\
		&\qquad\quad \Psi(X_t,1)\widehat\Psi(X_t,1)=p_{\mathcal B}(X)\label{eq:sb-pde-st-b},
	\end{align}
\end{subequations}
where $\Psi, \widehat\Psi\in (\mathbb R^d, [0,1])$ are the energy potentials and $\mathrm{Tr}(\cdot)$ is the matrix trace. From the above PDEs, we can use FPE and non-linear Feynman-Kac formula~\cite{Karatzas1987BrownianMA} to get an SB-inspired SDE:
\begin{equation}
	\hfill
	\begin{cases}
		dX_t=(f_t+g_tZ_t)dt+g_tdW_t \\
		dY_t=\frac 12 Z_t^\top Z_t dt +Z_t^\top dW_t \\
		d \hat Y_t=\left(\frac 12 \hat Z_t^\top \hat Z_t+\nabla_{X_t}\cdot (g_t\hat Z_t-f_t) +\hat Z_t^\top Z_t\right)dt+\hat Z_t dW_t
	\end{cases} \notag
\end{equation}
where $Z_t=g_t\nabla_{X_t} Y_t(X_t)=g_t\nabla_{X_t}\log \Psi(X_t,t)$ and $\hat Z_t=g_t\nabla_{X_t} \hat Y_t(X_t)=g_t\nabla_{X_t}\log \hat\Psi(X_t,t)$ are the SB optimal forward and backward drifts of the SDE.
Akin to the DPM SDE, we can still perform simulated SDE trajectories~\cite{chen2021likelihood} by defining any $f_t$ and $g_t$ which can ensure the convergence to the boundaries. However, simulating trajectories needs to run forward and backward in time and takes huge computational resources, meanwhile, it also requires two networks (\textit{i.e.,} forward and backward networks) simultaneously learning the forwarding and backwarding. Proofs of the derivations from SB problems to PDEs and SB-inspired SDEs can be found in the Supplementary.
Hence, developing a \textit{simulation-free} SB method is both practical and advantageous.

\section{Inverse Problem of Pansharpening and Learning Schr\"odinger Bridge Matching}

\begin{figure}[!t]
	\centering
	\includegraphics[width=\linewidth]{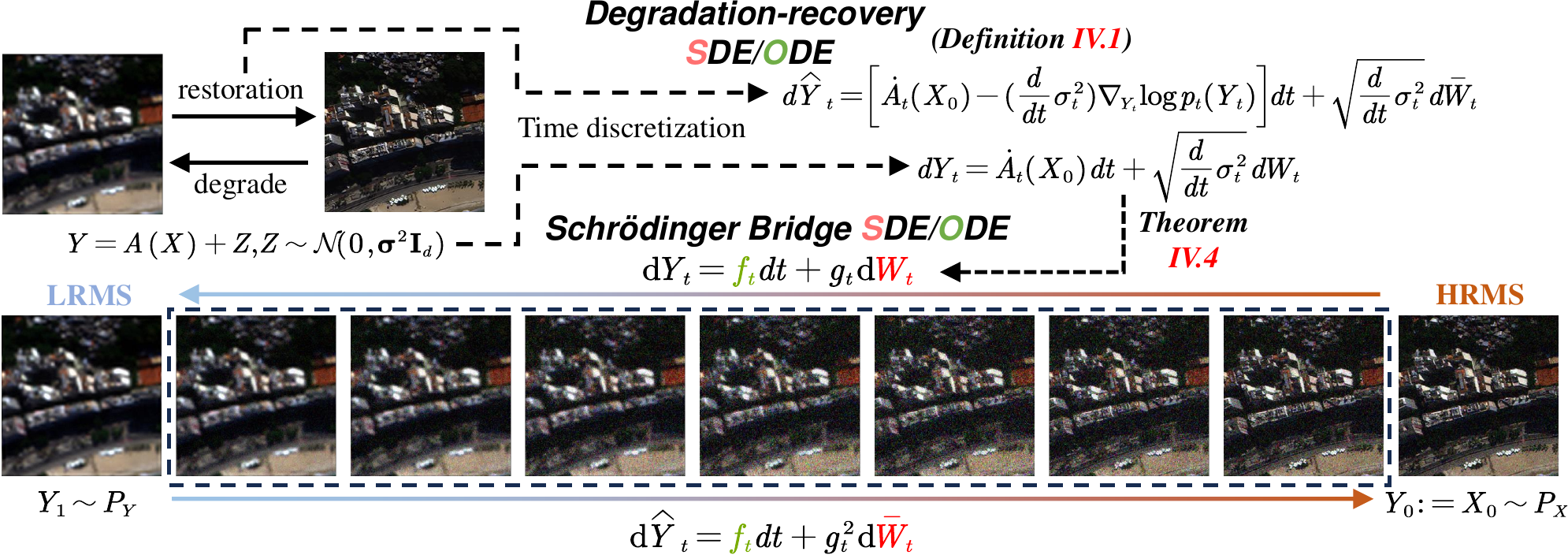}
	\caption{Main concept of the proposed SB SDE and ODE. HRMS is degraded by the forward SB SDE/ODE process (denoted as $q$ process) and recovered by the backward process (denoted as $p$ process). PAN, as the condition in the SB SDE or ODE, is omitted for a clear illustration.}
	\label{fig:main_conception}
\end{figure}

\subsection{Inverse Problems of Pansharpening}
\label{sec:inverse-pan}
First, we introduce the general formula of the inverse problem,
\begin{equation}
	Y=A(X)+Z, Z\sim \mathcal{N}(\mathbf 0, \mathbf \sigma^2 \mathbf I_d), \label{eq: deg}
\end{equation}
where $A(\cdot)$ is a degraded operator that can degrade the clean ground truth, and $Y$ is the measurement added by Gaussian noise with variance $\sigma^2$. For various inverse problems, such as superresolution~\cite{srdiff,saharia2022palette,saharia2022image,saharia2022photorealistic}, colorization~\cite{saharia2022palette,stable_diffusion}, MRI reconstruction~\cite{song2022solving,chung2023diffusion} et al, there are different degraded operators. Eq.~\eqref{eq: deg} also suits the pansharpening task. Recall the two widely-used traditional pansharpening methods (i.e., CS~\cite{cs1,cs2,cs3}, and MRA methods~\cite{mra1,mra2,mra3,mra4}):
\begin{subequations}
	\begin{align}
		\text{CS}:\quad &\mathbf M = \mathbf H - \mathbf g\odot(\mathbf {P^D}-\mathbf {I_L^D}),\\
		\text{MRA}:\quad &\mathbf M = \mathbf H - \mathbf g\odot(\mathbf{P^D}-\mathbf {P_L^D}),
	\end{align}
\end{subequations}
where $\mathbf{P^D}, \mathbf{P_L^D}, \mathbf{I_L^D}, \mathbf M,$ and $\mathbf H$ are the channel-duplicated PAN, channel-duplicated low-passed PAN, channel-duplicated intensity of LRMS, LRMS, HRMS images and $\mathbf g$ is a weighting vector.
From the above equations, a general degraded operator $A$ can be defined by considering $\mathbf M$ and $\mathbf H$ are measurement $Y$ and clean ground truth $X_0$,
\begin{equation}
	A(X|\mathbf{P, M})=X-\mathcal J(X, \mathbf{P}, \mathbf{M}), \label{eq:pan-deg}
\end{equation}
where $\mathcal J(\cdot)$ is a function and it can be noticed that $\mathbf H$ is the clean ground truth and $\mathbf M$ denotes the observation but $Z=0$ in the context of pansharpening. Without considering the Gaussian term $Z$, Eq.~\eqref{eq:pan-deg} tells us the pansharpening task be formulated as a typical inverse problem and different $A$s contribute different inverse problems.

\subsection{Degradation SDE and ODE of Pansharpening}
\label{sec:deg-ode-sde}
In the previous section, we represent pansharpening as the classical inverse problem formulation \eqref{eq: deg}. Next, we represent this process in the form of SDE in the following definition:
\begin{definition}[Stochastic degradation equation]
	\label{definition:stoc-deg-eq}
	The operator $A$ is time-dependent ($A_t$) and first-order differentiable. The same requirement applies to $\mathcal J$. The following SDE defines the degraded process by the operator $A_t$:
	\begin{equation}
		dY_t=\dot{A}_t(X_0)dt+\sqrt{\frac d {dt} \sigma_t^2}dW_t,\label{eq:pan-sde}
	\end{equation}
	where $Y_t$ is discrete in time, $\dot{A}_t$ is the time-dependent derivative of $A$, $\sigma_t^2$ is the variance of the Wiener process $dW_t$ at timestep $t$.
\end{definition}

Thus, we can formalize the process of pansharpening as an SDE. We refer to this process as ``degradation-SDE'' as shown in the upper panel in Fig.~\ref{fig:main_conception}, analogous to the forward process of diffusion in Eq.~\eqref{eq:sde}. Leveraging the FP equation, we can derive the reverse SDE formula:
\begin{equation}
	d \hat Y_t=\left[\dot A_t(X_0) -(\frac d {dt}\sigma_t^2)\nabla_{Y_t} \log p_t (Y_t)\right]dt+\sqrt{\frac d{dt}\sigma_t^2}d\bar W_t,
\end{equation}
where $d\bar W_t$ is the reverse process of the Wiener process~\cite{anderson1982reverse}.
The hat on a symbol represents the random variable when backwarding. This reverse SDE can be discretized using Euler-Maruyama discretization,
\begin{equation}
	\begin{aligned}
		&Y_{t-\Delta t}=Y_t+\underbrace{A_{t-\Delta t}(X_0)-A_t(X_0)}_{\text{recovery step}}\\
		&-\underbrace{(\sigma_{t-\Delta_t}^2-\sigma_t^2)\nabla_{Y_t}\log p_t (Y_t)}_{\text{denoising}}+\sqrt{\sigma_t^2-\sigma_{t-\Delta_t}^2} Z_t.
	\end{aligned}
\end{equation}
Similarly, we term it ``recovery-SDE''.
Naturally, we can derive the ODE forms for both the forward and backward processes:
\begin{subequations}
	\begin{align}
		&dY_t=\dot{A}_t(X_0)dt, \\
		d \hat Y_t=\dot A_t(X_0) &dt-\frac 12(\frac d {dt}\sigma_t^2)\nabla_{Y_t} \log p_t (Y_t). \label{eq:pan-ode}
	\end{align}
\end{subequations}
By parameterizing $\dot{A}_t$ with Gaussian transition kernels, we can directly train and sample the pansharpening task using Eqs.~\eqref{eq:pan-sde}-\eqref{eq:pan-ode}.

However, on the one hand, despite being demonstrated to exhibit good performance in the pansharpening task~\cite{deng22}, DPM often suffers from inefficient sampling due to the nature of Gaussian transition kernels. Especially when using specially customized DPM solvers~\cite{dpm_solver}, the generated quality tends to be low with a small number of sampling steps. In the following section, we demonstrate that this is attributed to the highly curved nature of DPM's sampling trajectory. Therefore, leveraging the properties of the Schr\"odinger Bridge (SB), we straighten the trajectory of DPM, reducing the training difficulty and simultaneously enhancing sampling efficiency.
On the other hand, we can intuitively sense that as $t\to 1$ (sample large enough times e.g., 1000), $Y_1$ can be regarded as a pure Gaussian noise. This implies that during sampling, we also need to initiate sampling from a Gaussian noise (see one endpoint of SDE in Fig.~\ref{fig:diff-scheme}), which is evidently unreasonable as we already possess prior information $\mathbf{M}$.
Hence, initiating sampling from the known prior is a reasonable and efficient approach. Clearly, existing DPM frameworks (\cite{ho2020denoising,saharia2022photorealistic,song2020score} and Definition~\ref{definition:stoc-deg-eq}) do not support it. A clear difference can be observed by comparing Fig.~\ref{fig:diff-scheme} and the lower panel in Fig.~\ref{fig:main_conception}.

\subsection{Simulation-free Schr\"odinger Bridge SDE}
\label{sect:simu-free-SB}
In Sect.~\ref{sec:deg-ode-sde}, we elucidate that the pansharpening task can be fully expressed in the forms of SDE and ODE. We utilized existing DPM frameworks for training and sampling. However, due to the existing shortcomings of DPM, applying it directly to the pansharpening task is both unreasonable and inefficient. To address these issues, we leverage the properties of SB and introduce SB ODE and SDE.

Since SB is known as an entropy-regulized optimal transport model that can be formulated as the following SDEs:
\begin{subequations}
	\begin{align}
		dY_t&=[f_t+g_t^2\nabla_{Y_t} \log p_t(Y_t)]dt+g_tdW_t, \label{pan-sb-f}\\
		d\hat Y_t&=[f_t-g_t^2\nabla_{Y_t} \log p_t(\hat Y_t)]dt+g_td\bar W_t,\label{pan-sb-b}
	\end{align}
\end{subequations}
where the two end-point distributions on the SB become $p_{X}$ and $p_{Y}$.
For Eqs.~\eqref{eq:pan-sde} and \eqref{pan-sb-f}, it can be observed that if we replace $\dot{A}_t$ with $f_t+g_t^2\nabla{Y_t} \log p_t(Y_t)$, we effortlessly transform the degradation-SDE into the form of SB SDE. Incorporating with Eq.~\eqref{eq:pan-deg}, we can further derive the following proposition:
\begin{proposition}\label{prop:inverse-to-sde}
	The pansharpening inverse problem can be represented as SB SDE forms\footnote{For simplicity of the symbols, the condition $\mathbf P$ is omitted.},
	\begin{subequations}
		\begin{equation}
			\begin{aligned}
				f_t&=1,\\
				\dot{\mathcal J}&=g_t^2\nabla_{Y_t} \log p_t(\hat Y_t),\label{eq:inverse-to-sde}
			\end{aligned}
		\end{equation} 
	\end{subequations}
	which is an implicit learning objective by setting $s_\theta:=\dot{\mathcal J}$. $s_\theta$ is a neural network, usually a U-Net~\cite{unet}.
	\label{cor: SB-SDE-pan}
\end{proposition}

\begin{proof}
	The proof of the Proposition~\ref{cor: SB-SDE-pan}  is sufficient.
\end{proof}

\begin{remark}
	Up to this point, it can be realized that the introduced SB-SDE can be seamlessly applied to the inverse problem's core, as we can employ a neural network to parameterize the degradation operator determined in the inverse problem (which may not be known). Moreover, relying on the multi-step diagram of the SDE, we can use the time steps as input (or conditions) to decompose a challenging single-step inverse degenerate process into multiple steps. This can alleviate the learning difficulty for the network, aligning with the observation in~\cite{DDIF}.
\end{remark}
However, in Eqs.~\eqref{pan-sb-f} and \eqref{pan-sb-b}, it is evident that these SDEs are non-linear. Moreover, the nonlinear term $\nabla_{Y_t} \log p_t(Y_t)$ has not been explicitly defined, making the solution of these SDEs challenging. Additionally, due to the presence of the nonlinear term, \textit{the trajectories of the SDE are not straight}, posing a challenge for sampling.
Recently, there has been research attention to the use of the Hopf-Cole transform~\cite{Hopf1950ThePD}, which enables the conversion of nonlinear terms into linear ones. However, these studies are confined to the sampling stage, transforming the PF ODE~\cite{song2019generative}, and the effectiveness of diffusion SDE remains unexplored.

Different from the previous diffusion works, we express SB in the form of linear SDEs.
\begin{theorem}
	\label{theorem:sb-sde}
	The Schrödinger Bridge problem in Eq.~\eqref{eq:sb-pde} can be expressed as the following forward-backward SDE with a linear $f_t$, which shares a form similar to that of a score-based model,
	\begin{subequations}
		\begin{align}
			&dY_t=f_tdt+g_tdW_t, &Y_0 \sim \hat \Psi(\cdot, 0), \label{pan-sb-f2}\\
			&d\hat Y_t=f_tdt+g_td\bar W_t, &Y_1\sim \Psi(\cdot, 1). \label{pan-sb-b2}
		\end{align}
	\end{subequations}
	The nonlinear terms in forward and backward drifts $\nabla_t\log \Psi, \nabla_t\log \hat \Psi$ have been absorbed into the initial conditions $\Psi(Y_t,0)$ and $\hat \Psi(Y_t,1)$.
\end{theorem}
\begin{proof}
	$dY_t=f_tdt+g_tdW_t$ shows that the probability change is an It$\hat{\text{o}}$ process, which can be reformulated by the FP equation,
	\begin{align}
		\frac{\partial p(Y,t)}{\partial t} = -\nabla\cdot (f_t p_t)+\frac 12 g_t \Delta p_t.
		\label{eq:sb-pf}
	\end{align}
	By comparing \eqref{eq:sb-pde} and \eqref{eq:sb-pf}, and using
	\begin{align}
		\mathrm{Tr}(\nabla_{Y_t} \Psi) = \Delta \Psi,	
	\end{align}
	it shows that the $\frac{\partial(Y_t, t)}{\partial t}$ has the same form when taking $\hat \Psi=p_t$. When reverse the SB PDE~\eqref{eq:sb-pde}:
	\begin{align}
		\begin{cases}
			\frac{\partial \Psi(Y,s)}{\partial s} = \nabla \Psi^\top f_s+\frac 12 \mathrm{Tr}(g_s^2 \nabla_{Y_s}^2 \Psi),\\
			\frac{\partial \widehat \Psi(x,t)}{\partial s} = \nabla\cdot (\widehat \Psi^\top f_s)-\frac 12 \mathrm{Tr}(g_s^2 \nabla_{Y_s}^2 \widehat \Psi),
		\end{cases} \label{eq:sb-pde-r}
	\end{align}
	where $s:=1-t$. This implies that $\Psi(Y,s)$ can be interpreted as the density of one SDE
	\begin{align}
		dY_s=-f_sds+g_sdW_s,
	\end{align}
	which is equal to Eq.~\eqref{pan-sb-b2}. In the rest of this paper, we set $t:=1-s$ under the continuous time SB context\footnote{We simultaneously define $X_s$ and $Y_t$ for a better description of SB in the context of the inverse problem. Such a definition is reasonable.}.
\end{proof}

It is easy to observe that we have directly removed the nonlinear term, absorbing it into the initial condition $\hat \Psi(\cdot, 0)$, leaving only completely linear terms and keeping the conclusion in Eqs.~\eqref{eq:inverse-to-sde} unchanged. Therefore, we can parameterize $\nabla_{Y_t}\log \hat\Psi(Y_t)$ as a network $s_\theta$, similar to training and sampling in DPM.

It's worth noting that our approach has a more appealing characteristic, i.e., it is simulation-free. Reviewing previous methods to solve the SB problems, such as the Iterative Proportional Fitting (IPF)~\cite{sinkhorn1964relationship} methods (\textit{a.k.a.} Sinkhorn algorithm), they require a simulation to regularize the entire trajectory of the SDE. Even recent deep transport methods (\textit{e.g.}, FB-SDE~\cite{chen2021likelihood}, Action Matching~\cite{neklyudov2022action}) still necessitate simulation, which is impractical for some high-dimensional data (\textit{e.g.}, multispectral or hyperspectral images) to simulate the entire trajectory and cache it. Our method aligns with score-based models in being simulation-free, and this is also why our approach can adapt to high-dimensional data.

\subsection{Schr\"odinger Bridge Matching for Pansharpening and its ODE version}
\label{sec:sb-sde-ode}

Similar to Probability Flow (PF)~\cite{lai2023fp,song2019generative} but distinct, we can derive an ODE with optimal transport properties from Eqs.~\eqref{eq:sb-pde}-\eqref{eq:sb-pde-st-b}. We designate this ODE as ``SB ODE''.
\begin{corollary} \label{cor:OT-ODE}
	When $g_t\to 0$, the proposed SDEs between $(X, Y)$ can be expressed by an ODE with the same posterior mean of SB SDE in Theorem~\ref{theorem:sb-sde} and exhibit the OT property:
	\begin{subequations}
		\begin{align}
			&dY_t=v_t(Y_t|X_0)dt,\\
			&dY_s=-v_t(Y_s|Y_t)ds,
		\end{align}
	\end{subequations}
\end{corollary}

It is worth noting that this ODE is not a PF ODE, and it simulates the OT plan~\cite{gushchin2022entropic,tong2023improving} when the corresponding stochastic terms in the SDE~\eqref{pan-sb-f2} vanish.
In this context, $v_t$ determines the speed of transition from $X_0$ to $Y_t$. If $v_t$ is simple and tractable, similar to the approach used in score-based models~\cite{song2019generative,song2020score}, we still can perform a effecient forward process and a stimulation-free train. This enables us to conveniently carry out the reverse ODE sampling process (sampling from $Y_t$ all the way back to $X_0$), aligning with our initial motivations.
We further elucidate in the Supplementary the relationship between our SB ODE and Flow Matching~\cite{lipman2023flow}, as well as its ability to derive a similar form to the Classifier Guidance~\cite{ho2022classifier} in DPMs. This paves the way for controlling image generation.

\subsection{Bridge Matching Parameterization, Train/Sampling Algorithms and Training Objectives}
\label{sect: bmp-train-sampling-algos.}

\SetCommentSty{myCommentStyle}
\newcommand\myCommentStyle[1]{\textcolor{green1}{#1}}
\SetKwComment{Comment}{$\triangleright$ }{}
\normalem
\begin{algorithm}[htbp]
	\caption{Training scheme}
	\label{algo:train}
	\KwIn{Undegraded samples $p_X$, degraded samples $p_Y$, other paired conditions $p_c$ (optional) and SB matching network $s_\theta$, learning rate $\eta$.}
	\KwOut{Trained SB SDE/ODE matching parameter $\theta$.}
	\While{until convergence}{
		\Comment{Construct time steps $t$ and two bridge endpoints $Y_0:=X_0$ and $Y_1$.}
		$t\sim \mathcal U(0, 1)$, $X_0\sim p_X, Y_1\sim p_Y, C\sim p_c;$\\
		\Comment{Sample samples on SB bridge (SB SDE~\eqref{eq:discri-forward-sde} or ODE~\eqref{eq: Y_t-in-ODE}).}
		$Y_t\sim \mathcal N(Y_t;\mu_t(Y_0, Y_1), \Sigma_t)$ or $Y_t=\frac{\hat\sigma_t^2}{\hat \sigma_t^2+\sigma_t^2} X_0+\frac{\sigma_t^2}{\hat \sigma_t^2+\sigma_t^2} Y_1;$\\
		\Comment{Take gradient step using one of four proposed bridge matching objectives.}
		$\theta \leftarrow \theta - \eta \nabla_{\theta} \mathcal  L(s_\theta(Y_t, C, t), Y_0, Y_1)$.
	}
\end{algorithm}
\begin{algorithm}[htbp]
	\caption{Sampling scheme}
	\label{algo:eval}
	\KwIn{Sampling timesteps $T$, trained $s_\theta$, degraded samples $p_0$, degraded samples $p_Y$, Conditional samples $p_c$ (optional).}
	\KwOut{Sampled data $X_0$.}
	\SetKwFunction{linspace}{linspace} 
	$T_s = \linspace(1,0,T)$ \Comment*[r]{Init timestep sequence.}
	$Y_1\leftarrow p_Y, C\leftarrow p_c$;\\
	\For{$t\leftarrow T_s$}{
		$\hat X_0\leftarrow s_\theta(Y_t, C, t)$\Comment*[r]{Predicted endpoint.}
		$Y_{t} \leftarrow p(Y_{t}|\hat X_0, Y_{t+\Delta t})$\Comment*[r]{Posterior sampling~\eqref{eq: sample}.} 
	}
	$X_0\leftarrow Y_0$.
\end{algorithm}

In this section, we perform necessary parameter design for SB SDE and SB ODE, including the design of $f_t$ and $g_t$ in the Eqs.~\eqref{pan-sb-f2} and~\eqref{pan-sb-b2}, as well as $\sigma_t$ in Eq.~\eqref{eq:ot-ode}.
We provide three feasible parameterization schemes, all of which can satisfy the assumptions of the SB system induced by Eqs.~\eqref{eq:sb-pde}-\eqref{eq:sb-pde-st-b}. Then, the different training objectives are derived and we further provide the training and sampling algorithms~\ref{algo:train} and \ref{algo:eval}. We follow the notation of \cite{ho2020denoising} by denoting the forward process as $q$ and backward process as $p$.

Since the derived SB SDE/ODE results in a non-Gaussian endpoint, we set the drift term $f_t:=0$ and further derive a forward process with an analytical solution. This approach resolves the computational intractability issue of Eqs.~\eqref{pan-sb-f} and \eqref{pan-sb-b}, as well as the problem of not converging to regions of high probability density caused by Eqs.~\eqref{pan-sb-f2} and \eqref{pan-sb-b2}.

\begin{proposition}[Analytic posterior when given bridge endpoints] When setting $f_t:=0$, $g_t:=\sqrt{\beta_t}$ (
	$\beta_t$ is a function that increases for $t\in [0, 1/2]$ and decrease $t\in (1/2,1]$), and given paired bridge endpoints $(X_0, Y_1)$, the posterior distribution has an analytic form:
	\label{prop:analtic-posterior}
	\begin{align}
		&q(Y_t|X_0, Y_1)=\mathcal N(Y_t;\mu_t(X_0, Y_1), \Sigma_t),\label{eq:discri-forward-sde} \\
		\mu_t=&\frac{\hat\sigma_t^2}{\hat \sigma_t^2+\sigma_t^2} X_0+\frac{\sigma_t^2}{\hat \sigma_t^2+\sigma_t^2} Y_1, \Sigma_t=\frac{\sigma_t^2\hat\sigma_t^2}{\hat\sigma_t^2+\sigma_t^2}\mathbf I, \label{eq: mean and variance}
	\end{align}
	where $\sigma_t^2:=\int_0^t \beta_\tau d_\tau$ and $\hat \sigma_t^2:=\int_t^1 \beta_\tau d\tau$ are variances of Gaussians. This analytic form avoids recursively running the poterior sampling in the above forward SDE:
	\begin{equation}
		q(Y_t|X_0, Y_1) = \int \prod_{k=n}^{N-1} p(Y_{k}|X_0,Y_{k+1})dY_{k+1},
	\end{equation}
	making the forward sampling much more effecient.
\end{proposition}

\begin{proof}
	Since we have $q(\cdot, t) = \Psi(\cdot, t)\hat\Psi(\cdot, t)$ contributed by Nelson's duality~\cite{nelson2020dynamical}, one can write,
	\begin{equation}
		q(Y_t|X_0,Y_1) = \Psi(Y_t, t|X_0)\hat\Psi(Y_t, t|Y_1).
	\end{equation}
	Due to fact that the solution of FP equations are actually $\Psi(Y_t, t|X_0)$ and $\hat\Psi(Y_t, t|Y_1)$, the posterior are two Guassian potential product:
	\begin{subequations}
		\begin{align}
			&\Psi(Y_t, t|X_0)\hat\Psi(Y_t, t|Y_1)\\
			&=\exp\left(-\frac 1 2 \left( \frac{\|Y_t-Y_0\|^2}{\sigma_t^2}+\frac{\|Y_t-Y_1\|^2}{\hat \sigma_t^2}\right)\right)\\
			&=\mathcal N(Y_t; \frac{\hat\sigma_t^2}{\hat \sigma_t^2+\sigma_t^2} X_0+\frac{\sigma_t^2}{\hat \sigma_t^2+\sigma_t^2} Y_1,\frac{\sigma_t^2\hat\sigma_t^2}{\hat\sigma_t^2+\sigma_t^2}\mathbf I),
		\end{align}
	\end{subequations}
	where $\sigma_t^2:=\int_0^t \beta_\tau d_\tau$ and $\hat \sigma_t^2:=\int_t^1 \beta_\tau d\tau$ are variances of Gaussians. Since the SDE should be discrelized, by slight abuse of notation as $Y_{t_n} = Y_n$, we prove $q(Y_t|X_0,Y_0)$ is the marginal probability density of the posterior $p(Y_n|X_0, Y_{n-1})$.
	When $f_t:=0$, $p(Y_n|X_0, Y_{n-1})$ has an analytic Gaussian form,
	\begin{subequations}
		\begin{align}
			&p(Y_n|X_0, Y_{n+1}) \label{eq: sample} \\
			=\mathcal N (Y_n; \frac{\alpha_n^2}{\alpha_n^2+\sigma_n^2}X_0& + \frac{\sigma_n^2}{\alpha_n^2+\sigma_n^2}Y_{n+1},\frac{\alpha_n^2\sigma_n^2}{\alpha_n^2+\sigma_n^2}\mathbf I),
		\end{align}
	\end{subequations}
	where $\alpha_n^2:=\int_{t_n}^{t_{n+1}} \beta_\tau d\tau$ is the accumulated Gassian variance. So, we can derive that,
	\begin{equation}
		q(Y_{N-1}|X_0, Y_{N}) = p(Y_{N-1}|X_0, Y_{N}),
	\end{equation}
	because one fact exists that $\alpha_{N-1}=\int_{t_{N-1}}^{t_N} \beta_\tau d\tau =\hat \sigma_{N-1}^2$.
	We can deduce at any timestep $t_n$, the relation still holds by induction:
	\begin{equation}
		\begin{aligned}
			q(Y_n|X_0,Y_N)=\int p(Y_n|X_0,Y_{n+1})q(Y_{n+1}|X_0,Y_{n+1}) dY_{n}
		\end{aligned}
	\end{equation}
	The RHS is a Gaussian that can be derived by using the Gaussian additive property:
	\begin{subequations}
		\begin{align}
			&\frac{\alpha_n^2}{\alpha_n^2+\sigma_n^2}X_0+\frac{\sigma_n^2}{\alpha_n^2+\sigma_n^2}\left(
			\frac{\hat{\sigma}_{n+1}^2}{\hat \sigma_{n+1}^2+\sigma_t^2}X_0
			+ \frac{\sigma_{n+1}^2}{\hat \sigma_{n+1}^2+\sigma_{n+1}^2}Y_N
			\right)\\
			&=\frac{\alpha_{n}^{2}(\hat{\sigma}_{n+1}^{2}+\sigma_{n+1}^{2})+\sigma_{n}^{2}\hat{\sigma}_{n+1}^{2}}{\sigma_{n+1}^{2}(\hat{\sigma}_{n}^{2}+\sigma_{n}^{2})}X_{0}+\frac{\sigma_{n}^{2}}{\hat{\sigma}_{n}^{2}+\sigma_{n}^{2}}Y_{N}  \\
			&= \frac{\alpha_{n}^{2}\sigma_{n+1}^{2}+\hat{\sigma}_{n+1}^{2}(\alpha_{n}^{2}+\sigma_{n}^{2})}{\sigma_{n+1}^{2}(\hat{\sigma}_{n}^{2}+\sigma_{n}^{2})}X_{0}+\frac{\sigma_{n}^{2}}{\hat{\sigma}_{n}^{2}+\sigma_{n}^{2}}Y_{N}  \\
			&= \frac{\hat{\sigma}_{n}^{2}}{\hat{\sigma}_{n}^{2}+\sigma_{n}^{2}}X_{0}+\frac{\sigma_{n}^{2}}{\hat{\sigma}_{n}^{2}+\sigma_{n}^{2}}Y_{N},
		\end{align}
	\end{subequations}
	By utilizing $\hat \sigma_n^2+\sigma_n^2$ is a const and
	\begin{align}
		\alpha_t=\sigma_{n+1}^2-\sigma_n^2=\hat \sigma_{n}^2 - \hat \sigma_{n+1}^2.
	\end{align}
	The variance $\Sigma_t$ of RHS is:
	\begin{subequations}
		\begin{align}
			&\frac{\alpha_{n}^{2}\sigma_{n}^{2}}{\alpha_{n}^{2}+\sigma_{n}^{2}}+\frac{\hat{\sigma}_{n+1}^{2}\sigma_{n+1}^{2}}{\hat{\sigma}_{n+1}^{2}+\sigma_{n+1}^{2}}\biggl(\frac{\sigma_{n}^{2}}{\alpha_{n}^{2}+\sigma_{n}^{2}}\biggr)^{2} \\
			&= \frac{\alpha_{n}^{2}\sigma_{n}^{2}(\hat{\sigma}_{n}^{2}+\sigma_{n}^{2})+\hat{\sigma}_{n+1}^{2}\sigma_{n}^{4}}{\sigma_{n+1}^{2}(\hat{\sigma}_{n}^{2}+\sigma_{n}^{2})}  \\
			&=\frac{\sigma_{n}^{2}\hat{\sigma}_{n}^{2}}{\hat{\sigma}_{n}^{2}+\sigma_{n}^{2}}.
		\end{align}
	\end{subequations}
	Thus, the proof is concluded.
\end{proof}
\noindent In practice, we set $\beta_t$ in the interval $[0, \frac 1 2]$ to be the quadratic function and flip it in $(\frac 1 2, 1]$,
\begin{align}
	\beta_t=
	\begin{cases}
		\left((\sqrt{\beta_{1/2}} - \sqrt{\beta_0}) t + \sqrt{\beta_0}\right)^2, &t\in [0, \frac 1 2]\\
		\left((\sqrt{\beta_{1/2}} - \sqrt{\beta_0})(1-t) + \sqrt{\beta_0} \right)^2, &t\in (\frac 1 2, 1]
	\end{cases}
\end{align}
where $\beta_0$ and $\beta_{1/2}$ are prefixed. Other SB hyperparameters are shown in Fig.~\ref{fig: schedule}.

Proposition~\ref{prop:analtic-posterior} provides an explicit bridge process of SB SDE. Similarly, we can write the analytic posterior of $v_t$ in Corollary~\ref{cor:OT-ODE}.
\begin{corollary}[Analytic posterior of SB ODE] \label{cor: SB ODE}
	The posterior of SB ODE in Corollary~\ref{cor:OT-ODE} has a closed-form posterior:
	\begin{equation}
		v_t(Y_t|X_0) = \frac {\beta_t} {\sigma_t^2}(Y_t-X_0),\label{eq:ot-ode}
	\end{equation}
	and the SB ODE can be discretized as,
	\begin{equation}
		Y_t=\frac{\hat\sigma_t^2}{\hat \sigma_t^2+\sigma_t^2} X_0+\frac{\sigma_t^2}{\hat \sigma_t^2+\sigma_t^2} Y_1. \label{eq: Y_t-in-ODE}
	\end{equation}
\end{corollary}
\begin{proof}
	This proof start to take the infinitestimal limit of $g_t^2:=\beta_t\to 0$ in Eq.~\eqref{pan-sb-f2}, and the variance of $q_t$ (\textit{i.e.,} $\frac{\alpha_n^2\sigma_n^2}{\alpha_n^2+\sigma_n^2}$) also converge to 0 as the nominator becomes to 0 faster. The process become deterministic in Eq.~\eqref{eq: Y_t-in-ODE} since the stochastic term is vanished. The mean $u_t$ is unchanged. From Eq.~\eqref{eq: mean and variance}, we can know the nonlinear term $\nabla_{Y_t} \log \hat \Psi(Y_t|X_0)$ absorbed in the drift term $f_t$ becomes,
	\begin{align}
		\nabla_{Y_t} \log \hat \Psi(Y_t|X_0) = \frac 1 {\sigma_t^2}(Y_t-X_0),
	\end{align}
	and we define $v_t$ with the nonlinear term multiplied by $\beta_t$,
	which is Eq.~\eqref{eq:ot-ode}.
\end{proof}

Clearly, from Eq.~\eqref{eq:ot-ode}, it can be seen that the ODE drift $v_t$ interpolates between $Y_1$ and $X_0$ We explicitly perform this interpolation process during the forward pass.
\begin{remark} 
	``SB ODE'' corresponds to the ODE versions of Eqs.~\eqref{pan-sb-f2} and \eqref{pan-sb-b2}. Its rationale lies in removing the stochastic terms from the SDE, transforming it into an ODE. Additionally, a favorable OT property~\cite{peyre2019computational,villani2009optimal} can be derived from Eq.~\eqref{eq:ot-ode}, making sampling extremely convenient and practical (or, in other words, ensuring that the curvature is flat~\cite{lee2023minimizing}).
\end{remark}
Moreover, when $\beta_t$ becomes a constant, SB ODE can have the same forward and backward process as Flow Matching~\cite{lipman2023flow}, which is detailed in discussion.

\begin{figure}[t]
	\centering
	\begin{subfigure}[t]{0.5\linewidth}
		\centering
		\includegraphics[width=\linewidth]{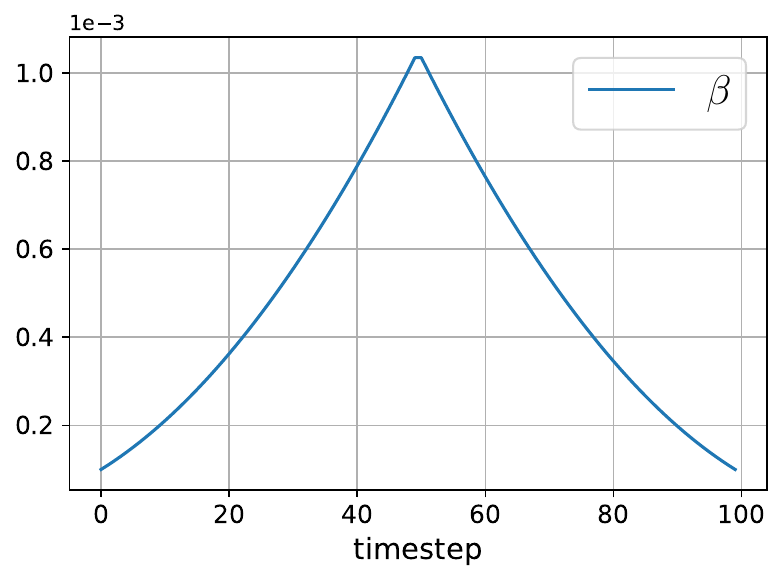}
		\caption{$\beta_t$ schedule.}
	\end{subfigure}%
	\begin{subfigure}[t]{0.5\linewidth}
		\centering
		\includegraphics[width=\linewidth]{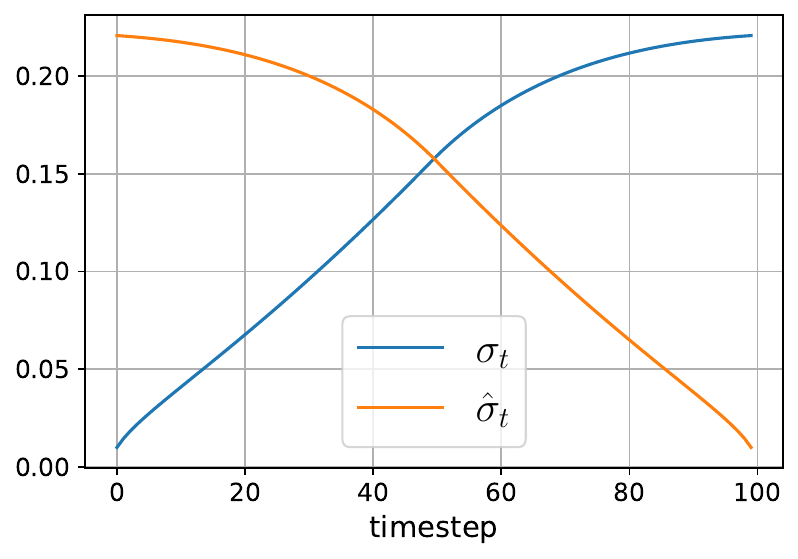}
		\captionsetup{width=\textwidth}
		\caption{\centering Standard deviations.}
	\end{subfigure}
	\begin{subfigure}[t]{0.5\linewidth}
		\centering
		\includegraphics[width=\linewidth]{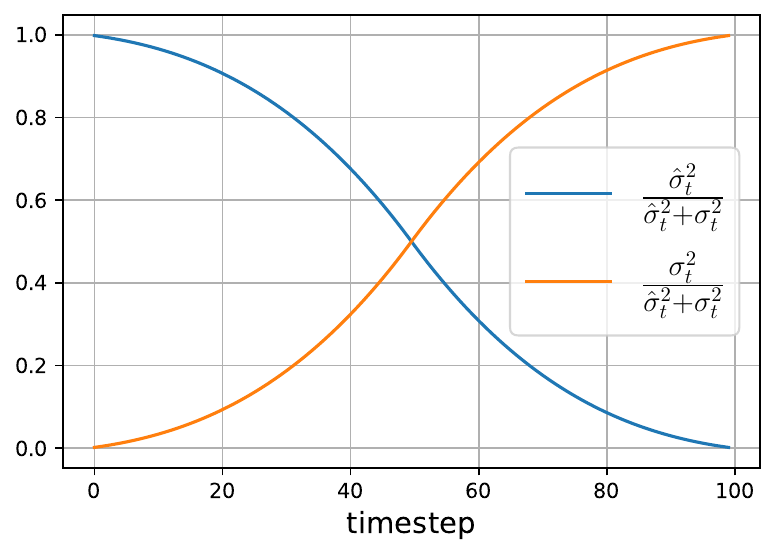}
		\caption{Mean schedule.}
	\end{subfigure}%
	\begin{subfigure}[t]{0.5\linewidth}
		\centering
		\includegraphics[width=\linewidth]{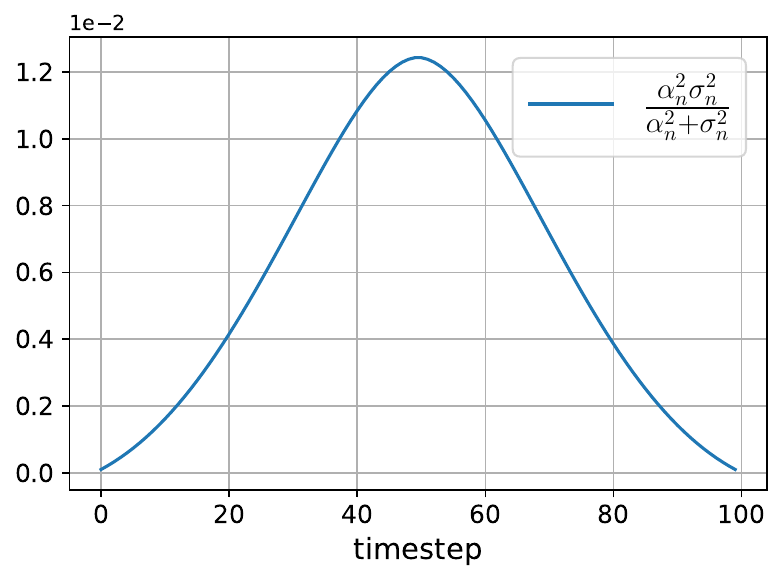}
		\captionsetup{width=\textwidth}
		\caption{\centering Variance schedule.}
		\label{fig:variance}
	\end{subfigure}
	
	\caption{$\beta_t$, $\mu_t$, $\sigma_t$ and variance schedules of the proposed SB matching.}
	\label{fig: schedule}
\end{figure}

To reverse back the SB forward process $q$ (\textit{i.e.}, $p$), we need to learn a neural network $s_\theta$ trained on a specific objective.
We provide four different learning objectives for SB SDE and ODE:
\begin{enumerate}
	\item[(a)] bridge endpoint $X_0$;
	\item[(b)] bridge length $Y_1-X_0$;
	\item[(c)] bridge posterior length $Y_t-X_0$;
	\item[(d)] bridge endpoint with score $[X_0, \epsilon]$ (SDE only).
\end{enumerate}
The first three objectives are suitable for both SB SDE and ODE but (d) only suits SB SDE as ODE does not involve a score function. To explain the rationale behind each objective, (a) directly predicts one bridge point which is similar to supervised learning, but the latter does not have a concept of timestep or dynamic; (b) lets the network learn the overall bridge length; while (c) straight-forwardly train the network by using the definition of $v_t$ in Eq.~\eqref{eq:ot-ode}; (d) incorporate score function in objective (a) which is related with the added Gaussian noise in $q$ process~\cite{somnath2023aligned} and we simplify it to $\epsilon$ inspired by \cite{ho2020denoising}.

\section{Bridge Matching Neural Architecture} \label{sect:bm-nn}

In the above sections, we elucidated the design paradigm, parameterization, training sampling strategy, and objectives of SB SDEs. However, we have not yet discussed how to design the network. In fact, the design of the network is a crucial aspect of SDE learning, and recent literature on score-based models and DPMs has also explored this topic~\cite{karras2023analyzing,peebles2023scalable,ma2024sit,crowson2024scalable}.
We design an efficient network architecture specifically for the pansharpening task, featuring high-speed training and inference. In comparison to architectures designed for image generation in DPMs (\textit{e.g.}, U-Net), our designed network exhibits superior performance.
Inspired by NAFNet~\cite{chen2022simple} in image restoration and Metaformer~\cite{yu2022metaformer} in architecture designing, we devise a NAFNet-like network named SBM-Net tailored for SB SDE learning.
\subsection{Overall Architecture}
\begin{figure}[t]
	\centering
    \includegraphics[width=0.7\linewidth]{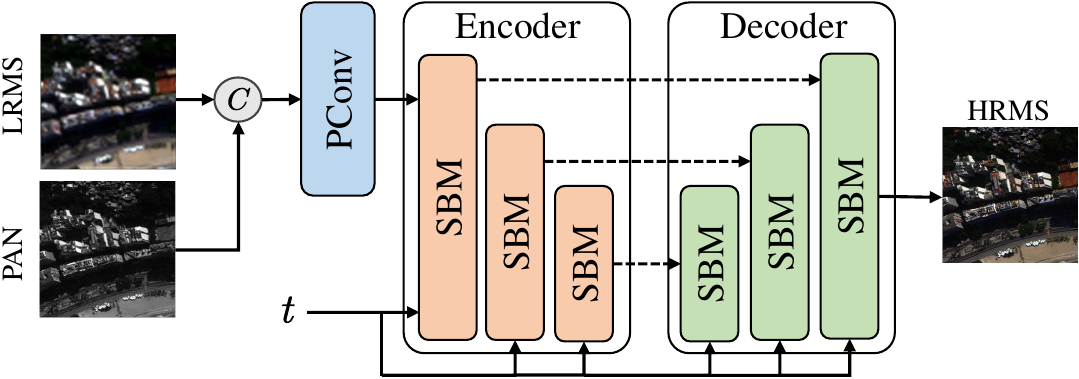}
		\caption{Overall architecture of the proposed SBM-Net for SB SDE learning. The SBM-Net is an encoder-decoder architecture. Timestep $t$ is injected into every SBM blocks. \textcircled{\footnotesize \raisebox{-.5pt}C} denotes tensor concatenation. Dash lines represent the U-Net-like shortcut connection. Valid lines are data flows.}
        \label{fig:arch}
\end{figure}

\begin{figure}[ht]
	\centering
	\includegraphics[width=0.6\linewidth]{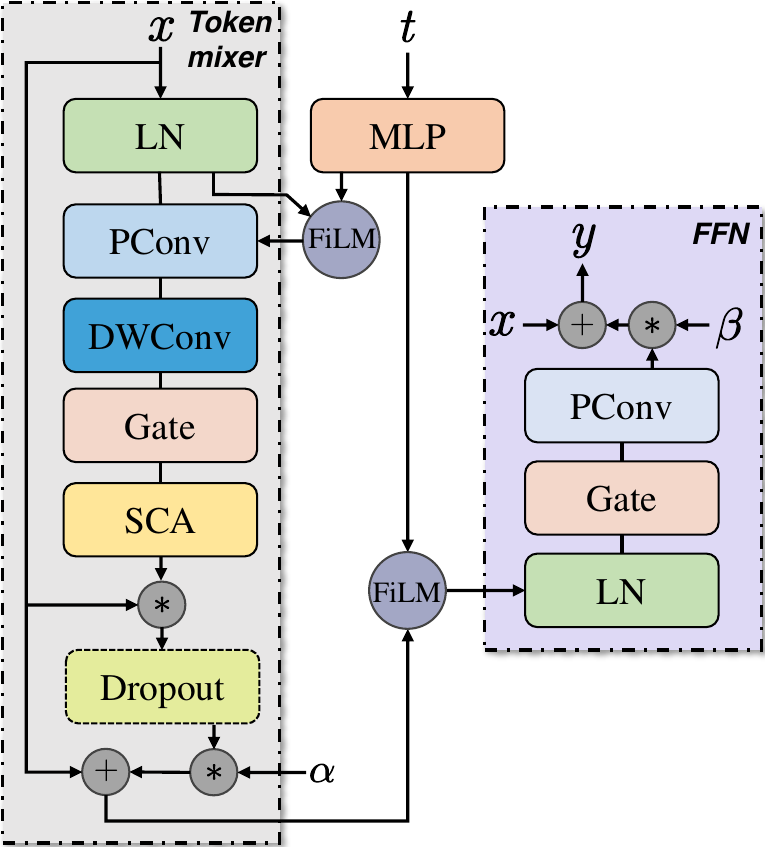}
	\caption{Illustration of the proposed efficient SBM block. SBM block shares a similar design scheme to the Metaformer block, composed of a token mixer and an FFN.}
	\label{fig:sbm-block}
\end{figure}

As shown in Fig.~\ref{fig:arch}, SBM-Net is an encoder-decoder architecture composed of several SBM blocks. The inputs LRMS and PAN are first concatenated along the channel dimension and then projected by a point-wise convolution into the latent space. Further, the projected feature is fed into the encoder layer by layer. Each of the SBM layers consists of several SBM blocks to extract spatial and channel representations. To incorporate the dependence of $t$, the timestep is embedded in the form of Sine-cosine~\cite{dosovitskiy2021an} to feed into the block. After processing every layer, the feature is downsampled, and doubled the channel dimensions. Then, the feature after the encoder is input to the decoder. The decoder has a similar architecture to the encoder. Every decoder's input is the input feature from the last SBM block and the feature processed by the corresponding encoder layer. We simply concatenate them together to feed into the decoder SBM layer. Similarly, the timestep is also fed into the decoder layer. After every decoder layer, the feature dimension is halved and upsampled by the pixel-shuffle operation~\cite{shi2016real}.
At last, another point-wise convolution is employed to project the feature back to the pixel space, forming the HRMS. In practice, we observed that forwarding the network twice can yield additional performance gains. Specifically, for each batch size of the input, we feed it into the network twice in a loop (i.e., the output of the first forward is fed back as input for the next forward). Note that objective (d) needs an additional predictive $\epsilon$, we double the number of output channels.

\subsection{SBM Block}
In the field of image generation, attention layers are often incorporated into U-Net architectures to facilitate additional inter-modal communication~\cite{saharia2022photorealistic,rombach2021highresolution} (\textit{e.g.}, between text and mask). In contrast, pansharpening, as a low-level image fusion technique in the domain of image fusion, lacks abstract high-level information such as text. Moreover, the quadratic memory overhead introduced by attention becomes unacceptable when dealing with high-resolution images under resource-constrained conditions.
Due to the conditioning on $t$ of $\nabla_{Y_t}\log p_t(Y_t)$, it is important to input timestep $t$ into the model. We feed $t$ into an MLP and adopt FiLM~\cite{perez2018film} to modulate feature $x$, formulated as,
\begin{subequations}
	\begin{align}
		&a,b=MLP(t)\\
		y&=(1+a)\times x + b
	\end{align}
\end{subequations}

Specifically, we remove the commonly-used Attention layer and replace it with a Simple Channel Attention (CSA), which can formulated as,
\begin{align}
	y = PConv(AdapPool(x)),
\end{align}
where $PConv$ denotes point-wise convolution and $AdapPool$ is adaptive pooling the input $x$ into $1\times 1\times c$ tensor. In order to reduce the number of parameters, we utilize a combination of point-wise convolution and depth-wise convoltuion~\cite{howard2017mobilenets} to replace the normal convolution. The pre-norm~\cite{nguyen2019transformers} scheme is adopted to stabilize the training. According to GELU~\cite{hendrycks2016gaussian} formulated, a simple gate is employed to replace the GELU activation function which is more computationally efficient and can be expressed as,
\begin{align}
	y = x_{[:c/2]} \times x_{[c/2:]},
\end{align}
where the input $x$ has the shape of $H \times W \times c$ ($c$ is always divisible by 2.). To facilitate the adapbility of the SBM block, we introduce another two learnable parameters $\alpha, \beta$ to control the block feature $z$ and shortcut feature $x$:
\begin{subequations}
	\begin{align}
		y=x+ \alpha \times z,\\
		y=x+ \beta \times z.
	\end{align}
\end{subequations}
The overall illustration of the SBM block is shown in Fig.~\ref{fig:sbm-block}.

\section{Experiments}
\label{sect:exp}

In this section, we first validate the designed training and sampling algorithms on a toy example, revealing some properties of SB SDEs and ODEs. Then we elaborate on the pansharpening experimental settings, datasets, and benchmarking. Performances of different traditional methods, regressive methods, and diffusion (or score)-based methods are compared on pansharpening reduced and full assessments. At last, some ablation studies are included to further elucidate some possible factors that can affect the proposed SB SDE learning.

\subsection{Toy Example}

We first examine the proposed SB SDE on a toy dataset. We set an 8 Gaussian mixture distribution as the SB start point $P_{X}$ and a Swiss roll distribution as the SB endpoint $P_Y$. The learned SB can translate the Gaussian mixture to the Swiss Roll as shown in Fig.~\ref{fig: toy-example}. The proposed SB SDE can \textit{progressively} transport the 8 Gaussian mixture to the Swiss Roll distribution, which verifies the efficacy.

\begin{figure}[htbp]
	\includegraphics[width=\linewidth]{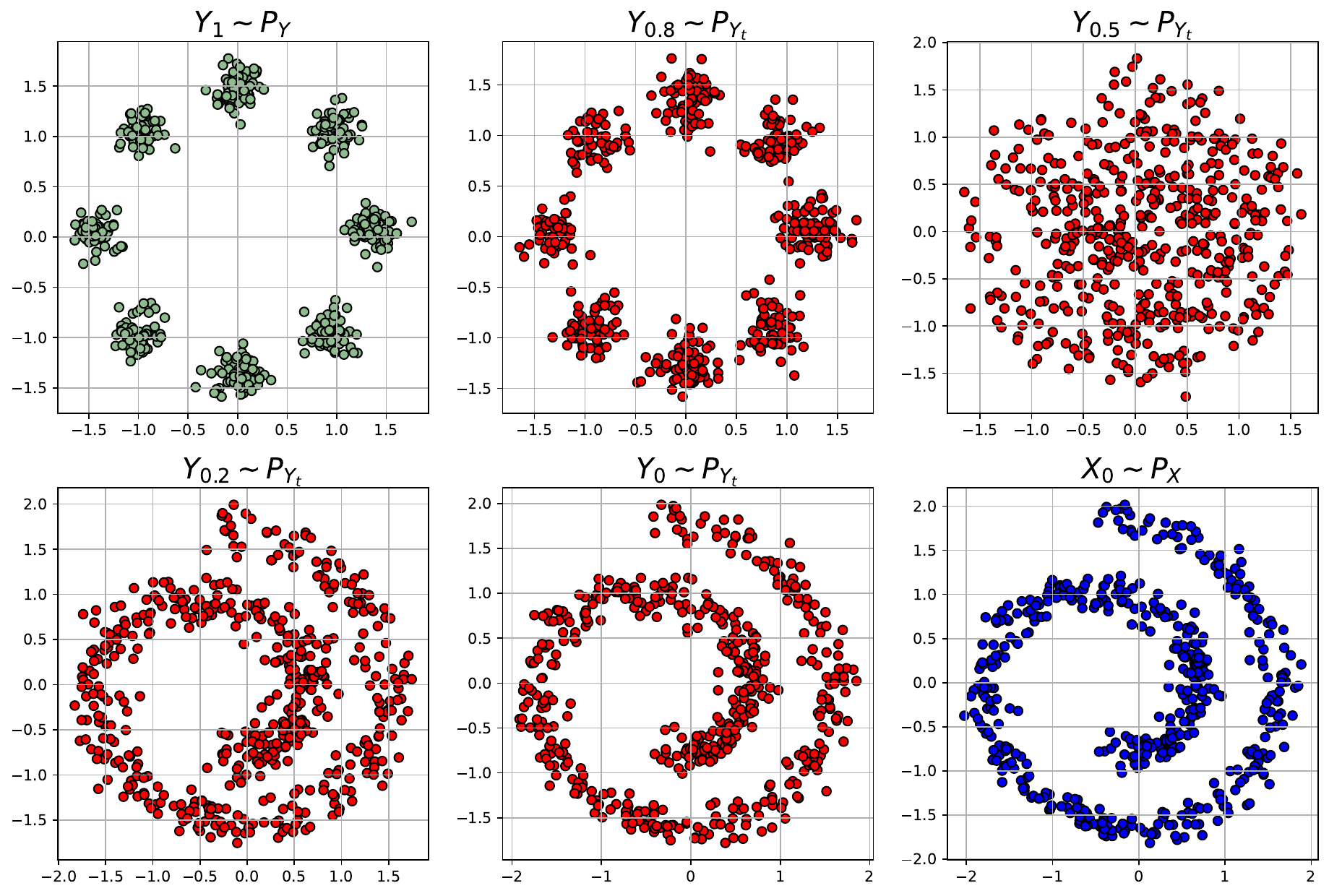}
	\caption{2D toy example: 8 Gaussian Mixture translates to Swiss Roll.}
	\label{fig: toy-example}
\end{figure}

\subsection{Experiment Setup \& Implementation Details}

In the following experiment, we conduct pansharpening tasks using SB SDE and SB ODE and compare them with PanDiff~\cite{pandiff}, which is based on vanilla DPM~\cite{ho2020denoising} as the baseline. 
We leverage the algorithm~\ref{algo:train} and \ref{algo:eval} to train and sample from the dataset. $p_X$ is the GT (\textit{i.e.,} HRMS), $p_Y$ is the LRMS, and $p_c$ is the PAN as the condition for the proposed SDE and ODE.
During the sampling phase, we employ \textit{5-step} Euler SDE sampling and \textit{1-step} ODE sampling for all datasets, in order to compare the sampling performance of other DPM methods. We will discuss the advantages and disadvantages of SB SDE and ODE and some of their features.


Regarding the network architecture proposed in Sect.~\ref{sect:bm-nn}, we set the number of channels for the first SBM layer to 32, and each SBM layer contains 3 SBM blocks. In the encoder, we use the $3\times 3$ convolution for downsampling and apply pixel-shuffle~\cite{shi2016real} for upsampling in the decoder. The input LRMS and PAN images are both normalized to the range of 0 to 1.

\newcommand{\best}[1]{{\color{red} \textbf{#1}}}
\newcommand{\second}[1]{{\color{blue} \textbf{#1}}}
\begin{table}[!ht]
	\centering
	\setlength{\tabcolsep}{3pt}
	\renewcommand\arraystretch{1.05}
	\caption{Quantitative results of all competing methods. The Q2n index is Q8/Q4 for 8-band/4-band data. The best results are in \best{red} and the second best results are in \second{blue}.}
	\label{tab:pansharpening}
	\resizebox{\linewidth}{!}{
		\begin{tabular}{clcccc|ccc}
			\toprule
			\hline
			\multicolumn{2}{l}{\multirow{2}{*}{}} &
			\multicolumn{4}{c|}{Reduced Resolution (RR): Avg$\pm$std} &
			\multicolumn{3}{c}{Full Resolution (FR): Avg$\pm$std} \\ \cline{3-9}
			\multicolumn{2}{l}{} &
			SAM & 
			ERGAS & 
			Q2n & 
			SCC & 
			$D_\lambda$ & 
			$D_s$ & 
			HQNR \\ 
			\multirow{16}{*}{\begin{tabular}[c]{@{}c@{}}
					WorldView-3 \\ (WV3, 8-band)
			\end{tabular}} &
			BDSD-PC~\cite{bdsd-pc} &
			5.47$\pm$1.72 &
			4.65$\pm$1.47 &
			0.812$\pm$0.106 &
			0.905$\pm$0.042 &
			0.063$\pm$0.024 &
			0.073$\pm$0.036 &
			0.870$\pm$0.053 \\
			&
			MTF-GLP-FS~\cite{mtf-glp-fs} &
			5.32$\pm$1.65 &
			4.65$\pm$1.44 &
			0.818$\pm$0.101 &
			0.898$\pm$0.047 &
			0.021$\pm$0.008 &
			0.063$\pm$0.028 &
			0.918$\pm$0.035 \\
			&
			BT-H~\cite{BT-H} &
			4.90$\pm$1.30 &
			4.52$\pm$1.33 &
			0.818$\pm$0.102 &
			0.924$\pm$0.024 &
			0.057$\pm$0.023 &
			0.081$\pm$0.037 &
			0.867$\pm$0.054 \\
			& 
			LRTCFPan~\cite{LRTCFPan} & 4.74$\pm$1.41& 4.32$\pm$1.44& 0.846$\pm$0.091 & 0.927$\pm$0.023 & {0.018$\pm$0.007}& 0.053$\pm$0.026& 0.931$\pm$0.031 \\
			\cline{2-9}
			&
			PNN~\cite{pnn} &
			3.68$\pm$0.76 &
			2.68$\pm$0.65 &
			0.893$\pm$0.092 &
			0.976$\pm$0.008 &
			0.021$\pm$0.008 &
			0.043$\pm$0.015 &
			0.937$\pm$0.021 \\
			&
			PanNet~\cite{pannet} &
			3.62$\pm$0.77 &
			2.67$\pm$0.69 &
			0.891$\pm$0.093 &
			0.976$\pm$0.009 &
			0.017$\pm$0.007 &
			0.047$\pm$0.021 &
			0.937$\pm$0.027 \\
			&
			MSDCNN~\cite{msdcnn} &
			3.78$\pm$0.80 &
			2.76$\pm$0.69 &
			0.890$\pm$0.090 &
			0.974$\pm$0.008 &
			0.023$\pm$0.009 &
			0.047$\pm$0.020 &
			0.932$\pm$0.027 \\
			&
			DiCNN~\cite{dicnn} &
			3.59$\pm$0.76 &
			2.67$\pm$0.66 &
			0.900$\pm$0.087 &
			0.976$\pm$0.007 &
			0.036$\pm$0.011 &
			0.046$\pm$0.018 &
			0.920$\pm$0.026 \\
			&
			FusionNet~\cite{fusionnet} &
			3.33$\pm$0.70 &
			2.47$\pm$0.64 &
			0.904$\pm$0.090 &
			0.981$\pm$0.007 &
			0.024$\pm$0.009 &
			{0.036$\pm$0.014} &
			{0.941$\pm$0.020} \\
			&
			LAGConv~\cite{lagconv} &
			3.10$\pm$0.56 &
			2.30$\pm$0.61 &
			0.910$\pm$0.091 &
			0.984$\pm$0.007 &
			0.037$\pm$0.015 &
			0.042$\pm$0.015 &
			0.923$\pm$0.025 \\
			&
			Invformer~\cite{panformer} & 3.25$\pm$0.64 & 2.39$\pm$0.52 & 0.906$\pm$0.084 & 0.983$\pm$0.005 & 0.055$\pm$0.029 & 0.068$\pm$0.031 & 0.882$\pm$0.049 \\
			&
			DCFNet~\cite{dcfnet} &
			{3.03$\pm$0.74} &
			{2.16$\pm$0.46} &
			0.905$\pm$0.088 &
			{0.986$\pm$0.004} &
			0.078$\pm$0.081 &
			0.051$\pm$0.034 &
			0.877$\pm$0.101 \\
			&
			HMPNet~\cite{hmpnet} & 3.06$\pm$0.58 & 2.23$\pm$0.55 & {0.916$\pm$0.087} & 0.986$\pm$0.005 & 0.018$\pm$0.007 & 0.053$\pm$0.006 & 0.929$\pm$0.011 \\
			\cline{2-9}
			&
			PanDiff~\cite{pandiff}
			& 3.30$\pm$0.60
			& 2.47$\pm$0.58
			& 0.898$\pm$0.088
			& 0.980$\pm$0.006
			& 0.027$\pm$0.012
			& 0.054$\pm$0.026
			& 0.920$\pm$0.036
			\\
			& 
			DDIF~\cite{DDIF}
			& \second{2.74$\pm$0.51} 
			& \second{2.02$\pm$0.45} 
			& \second{0.920$\pm$0.082}
			& \second{0.988$\pm$0.003}
			& 0.026$\pm$0.018
			& \best{0.023$\pm$0.008}
			& \second{0.952$\pm$0.017}
			\\
			&
			Proposed(\textit{5-step SDE}) &
			\best{2.73$\pm$0.51} &
			\best{1.99$\pm$0.44} &
			\best{0.921$\pm$0.081} &
			\best{0.989$\pm$0.003} &
			\second{0.013$\pm$0.005} &
			0.039$\pm$0.005 &
			0.949$\pm$0.009
			\\
			&
			Proposed(\textit{1-step ODE}) &
			3.67$\pm$0.81 &
			2.74$\pm$0.74 &
			0.886$\pm$0.096 &
			0.974$\pm$0.009 &
			\best{0.011$\pm$0.004} &
			\second{0.031$\pm$0.002} &
			\best{0.958$\pm$0.005}
			\\
			
			\hline
			\multirow{14}{*}{\begin{tabular}[c]{@{}c@{}}GaoFen2 \\ (GF2, 4-band)\end{tabular}} &
			BDSD-PC~\cite{bdsd-pc} &
			1.71$\pm$0.32 &
			1.70$\pm$0.41 &
			0.993$\pm$0.031 &
			0.945$\pm$0.017 &
			0.076$\pm$0.030 &
			0.155$\pm$0.028 &
			0.781$\pm$0.041 \\
			&
			MTF-GLP-FS~\cite{mtf-glp-fs} &
			1.68$\pm$0.35 &
			1.60$\pm$0.35 &
			0.891$\pm$0.026 &
			0.939$\pm$0.020 &
			0.035$\pm$0.014 &
			0.143$\pm$0.028 &
			0.823$\pm$0.035 \\
			&
			BT-H~\cite{BT-H} &
			1.68$\pm$0.32 &
			1.55$\pm$0.36 &
			0.909$\pm$0.029 &
			0.951$\pm$0.015 &
			0.060$\pm$0.025 &
			0.131$\pm$0.019 &
			0.817$\pm$0.031 \\
			& LRTCFPan~\cite{LRTCFPan} & 1.30$\pm$0.31& 1.27$\pm$0.34& 0.935$\pm$0.030& 0.964$\pm$0.012 & 0.033$\pm$0.027& 0.090$\pm$0.014& 0.881$\pm$0.023
			\\
			\cline{2-9}
			&
			PNN~\cite{pnn} &
			1.05$\pm$0.23 &
			1.06$\pm$0.24 &
			0.960$\pm$0.010 &
			0.977$\pm$0.005 &
			0.037$\pm$0.029 &
			0.094$\pm$0.022 &
			0.873$\pm$0.037 \\
			&
			PanNet~\cite{pannet} &
			1.00$\pm$0.21 &
			0.92$\pm$0.19 &
			0.967$\pm$0.010 &
			0.983$\pm$0.004 &
			{0.021$\pm$0.011} &
			0.080$\pm$0.018 &
			0.901$\pm$0.020 \\
			&
			MSDCNN~\cite{msdcnn} &
			1.05$\pm$0.22 &
			1.04$\pm$0.23 &
			0.961$\pm$0.011 &
			0.978$\pm$0.005 &
			0.027$\pm$0.013 &
			0.073$\pm$0.009 &
			0.902$\pm$0.013 \\
			&
			DiCNN~\cite{dicnn} &
			1.05$\pm$0.23 &
			1.08$\pm$0.25 &
			0.959$\pm$0.010 &
			0.977$\pm$0.006 &
			0.041$\pm$0.012 &
			0.099$\pm$0.013 &
			0.864$\pm$0.017 \\
			&
			FusionNet~\cite{fusionnet} &
			0.97$\pm$0.21 &
			0.99$\pm$0.22 &
			0.964$\pm$0.009 &
			0.981$\pm$0.005 &
			0.040$\pm$0.013 &
			0.101$\pm$0.013 &
			0.863$\pm$0.018 \\
			&
			LAGConv~\cite{lagconv} &
			{0.78$\pm$0.15} &
			0.69$\pm$0.11 &
			0.980$\pm$0.009 &
			0.991$\pm$0.002 &
			0.032$\pm$0.013 &
			0.079$\pm$0.014 &
			0.891$\pm$0.020 \\
			&
			Invformer~\cite{panformer} & 0.83$\pm$0.14 & 0.70$\pm$0.11 & 0.977$\pm$0.012 & 0.980$\pm$0.002 & 0.059$\pm$0.026 & 0.110$\pm$0.015 & 0.838$\pm$0.024 \\
			&
			DCFNet~\cite{dcfnet} &
			0.89$\pm$0.16 &
			0.81$\pm$0.14 &
			0.973$\pm$0.010 &
			0.985$\pm$0.002 &
			0.023$\pm$0.012 &
			{0.066$\pm$0.010} &
			{0.912$\pm$0.012} \\
			&
			HMPNet~\cite{hmpnet} & 0.80$\pm$0.14 & \second{0.56$\pm$0.10} & {0.981$\pm$0.030} & \second{0.993$\pm$0.003} & 0.080$\pm$0.050 & 0.115$\pm$0.012 & 0.815$\pm$0.049\\
			\cline{2-9}
			&
			PanDiff~\cite{pandiff}
			& 0.89$\pm$0.12
			& 0.75$\pm$0.10
			& 0.979$\pm$0.010
			& 0.989$\pm$0.002
			& 0.027$\pm$0.020
			& 0.073$\pm$0.010
			& 0.903$\pm$0.021
			\\
			&
			DDIF~\cite{DDIF}
			& \second{0.64$\pm$0.12}
			& 0.57$\pm$0.10
			& \second{0.986$\pm$0.008}
			& 0.986$\pm$0.004
			& \second{0.020$\pm$0.011}
			& 0.041$\pm$0.010
			& 0.940$\pm$0.014
			\\
			&
			Proposed(\textit{5-step SDE}) & 
			\best{0.59$\pm$0.11} &
			\best{0.53$\pm$0.10} &
			\best{0.987$\pm$0.007} &
			\best{0.994$\pm$0.002} &
			0.026$\pm$0.012 &
			\second{0.028$\pm$0.010} &
			\second{0.947$\pm$0.013}
			\\
			&
			Proposed(\textit{1-step ODE}) &
			0.80$\pm$0.15 &
			0.72$\pm$0.11 &
			0.978$\pm$0.009 &
			0.988$\pm$0.002 &
			\best{0.016$\pm$0.012} &
			\best{0.017$\pm$0.004} &
			\best{0.967$\pm$0.010}
			\\

			\hline
			\multicolumn{1}{l}{} &
			Ideal value &
			\textbf{0} &
			\textbf{0} &
			\textbf{1} &
			\textbf{1} &
			\textbf{0} &
			\textbf{0} &
			\textbf{1} \\ 
			\hline
			\bottomrule
		\end{tabular}%
	}
\end{table}

\subsection{Datasets}
We conduct experiments on a standard pansharpening data-collection (\textit{i.e.}, Pancollection\footnote{\url{ https://liangjiandeng.github.io/PanCollection.html}}), which includes WorldView-3 (WV3, 8 bands) and GaoFen-2 (GF2, 4band) data. The reduced data are simulated from real-world images using Wald's protocol~\cite{wald1997fusion}. WV3 dataset contains 9714/1080 samples for training and validation. Each sample consists of a PAN/LRMS/GT image pair of size $64\times 64\times 1$, $16\times 16\times 8$, and $64\times 64\times 8$, respectively. PAN image has a spatial resolution of 0.3m, whereas the LRMS image has a spatial resolution of 1.2m. GF2 dataset contains
19809/2201 samples for training and validation.
Each sample consists of a PAN/LRMS/GT image
pair of sizes $64\times 64\times 1, 16\times 16\times 4$, and $64\times 64\times 4$,
respectively. PAN images have a spatial resolution
of 0.8m, while LRMS images have a spatial
resolution of 3.2m. To evaluate the performance, we perform the reduced-resolution and full-resolution
experiments to compute the reference and non-reference metrics, respectively. The WV3 reduced-resolution test set has 20 PAN/LRMS/GT image pairs of size $256\times 256\times 1, 64\times 64\times 8$ and $256\times 256\times 8$. The WV3 full-resolution test set consists of a PAN/LRMS image pair of size  $512\times 512\times 1, 128\times 128\times 8$. The GF2 reduced and full-resolution test sets have the same samples as the WV3 test set but only differ in the number of bands (4 bands).

To verify the performance of each method on hyperspectral real remote sensing data, we utilize the GF5-GF1 public dataset~\cite{guo2022deep}. The GF5-GF1 dataset contains HSIs and MSIs, where the spatial size of MSIs is twice that of HSIs (i.e., $1161\times 1120\times 150$ for HSIs and $2332\times 2258\times 4$ for MSIs). We randomly cropped HSIs and MSIs into patches of size $40\times 40$ and $80\times 80$ with an overlap of 10 and 20, respectively, to generate real data. Based on the same patching scheme, furthermore, we can get the HRHSI ($80\times 80$) and HRMSI ($160\times 160$) patches for simulated data. We applied the provided modulated transfer functions
(MTFs) to the patched HRHSI and the patched HRMSI following Wald's protocol. To get the final HRMSI, we adjusted the simulated HRMSI by using the modified $\mathbf{ M = (M -B\cdot R)/A}$ (as proposed in \cite{guo2022deep}), where $\mathbf A$ and $\mathbf B$ are the correlated weight tensors, and $\mathbf R$ is the spectral response function. Finally, we obtained 150 simulated LRHSI and HRMSI patches with sizes $40\times 40\times 150$ and $80\times 80\times 4$, respectively, with the original LRHSI serving as ground truth. We divided the 150 patches into train/validation/test sets using the following percentages
80\%/10\%/10\%.

\subsection{Benchmarking}
\newcommand{\methodcomp}[1]{\textsuperscript{\textit{#1}}}
For the WV3 and GF2 multispectral datasets, we compare the proposed SB SDE/ODE with representative methods such as CS, MRA, VO, regression-based deep methods, and recent DPM-based methods as follows:
\begin{enumerate}
	\item Model-based methods: BDSD-PC\methodcomp{TGRS, 2017}~\cite{bdsd-pc}, MTF-GLP-FS\methodcomp{TIP, 2018}~\cite{mtf-glp-fs}, BT-H\methodcomp{GRSL, 2017}~\cite{BT-H}, and LRTCFPan\methodcomp{TIP, 2023}~\cite{LRTCFPan};
	\item Regressive DL-based methods: PNN\methodcomp{RS, 2016}~\cite{pnn}, PanNet\methodcomp{ICCV, 2017}~\cite{pannet}, MSDCNN\methodcomp{JSTAR, 2018}~\cite{msdcnn}, DiCNN\methodcomp{JSTAR, 2019}~\cite{dicnn}, FusionNet\methodcomp{TGRS,2020}~\cite{fusionnet}, LAGConv\methodcomp{AAAI,2022}~\cite{lagconv}, Invformer\methodcomp{AAAI, 2022}~\cite{ctinn}, DCFNet\methodcomp{ICCV, 2021}~\cite{dcfnet}, HMPNet\methodcomp{TNNLS, 2023}~\cite{hmpnet};
	\item DPM (score)-based methods: PanDiff\methodcomp{TGRS, 2023}~\cite{pandiff},
	DDIF\methodcomp{InFus, 2024}~\cite{DDIF}.
\end{enumerate}

For the hyperspectral real GF5-GF1 dataset, we choose widely-applied model-based methods, DL-based regressive methods, and DPM (score)-based methods to compare, listed as follows:
\begin{enumerate}
	\item Model-based methods: CNMF\methodcomp{IGRSS,2011}~\cite{CNMF}, Hysure\methodcomp{TGRS, 2014}~\cite{Hysure}, GSA\methodcomp{TGRS, 2007}~\cite{GSA}, LTTR\methodcomp{TNNLS, 2019}~\cite{LTTR}, LTMR\methodcomp{TIP, 2019}~\cite{LTMR}, MTF-GLP-HS\methodcomp{JSATR, 2015}~\cite{mtf-glp-hs};
	\item Regessive DL-based methods: ResTFNet\methodcomp{InFus, 2020}~\cite{ResTFNet}, SSRNet\methodcomp{TGRS, 2020}~\cite{SSRNet}, Fusformer\methodcomp{GRSL, 2022}~\cite{hu2022fusformer}, HSRNet\methodcomp{TNNLS, 2021}~\cite{HSRNet}, DHIF\methodcomp{TCI, 2022}~\cite{DHIF};
	\item DPM (score)-based methods: PanDiff\methodcomp{TGRS, 2023}~\cite{pandiff},
	DDIF\methodcomp{InFus, 2024}~\cite{DDIF}.
\end{enumerate}
Following \cite{deng22}, we employ SAM~\cite{sam}, ERGAS~\cite{ergas}, Q2n~\cite{q2n}, and SCC~\cite{scc} metrics to validate the effectiveness of different methods on the WV3 and GF2 reduced-resolution datasets. For the GF5-GF1 dataset with a greater number of spectral bands, we include PSNR and SSIM metrics. For the full-resolution experiments, we use $D_\lambda$, $D_s$, and HQNR~\cite{HQNR} as the evaluation metrics.
Note that the same training and data augmentation strategy is applied to the compared methods for a fair comparison.

\subsection{Reduced Assessments}
\begin{figure}[ht]
	\centering
	\includegraphics[width=\linewidth]{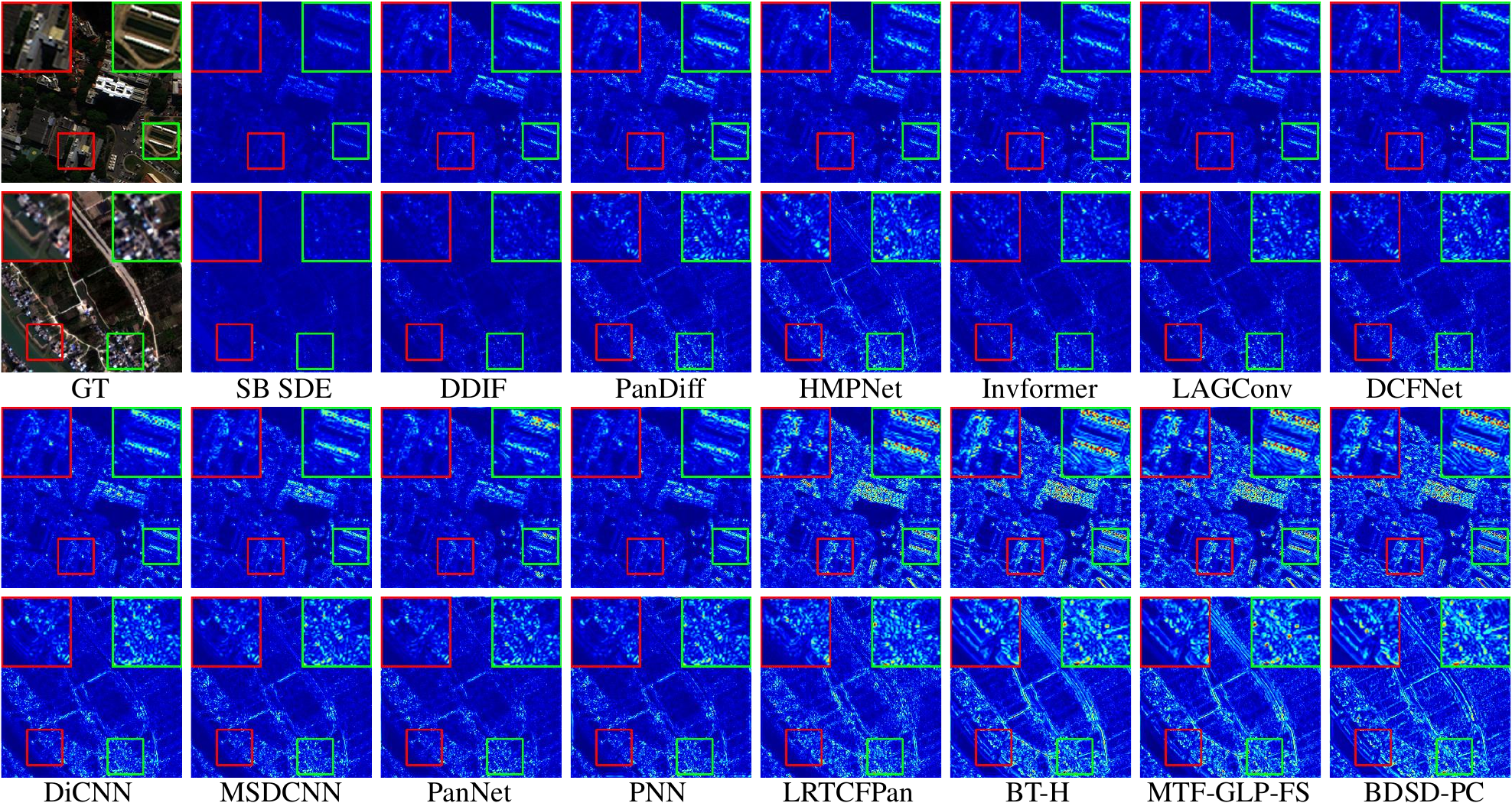}
	\caption{Illustration of ground truth and error maps with various pansharpening methods on WV3 (Row 1 and 3) and GF2 (Row 2 and 4) reduced-resolution test set.}
	\label{fig:wv3-reduced-comp}
\end{figure}

\begin{figure}[ht]
	\centering
	\includegraphics[width=\linewidth]{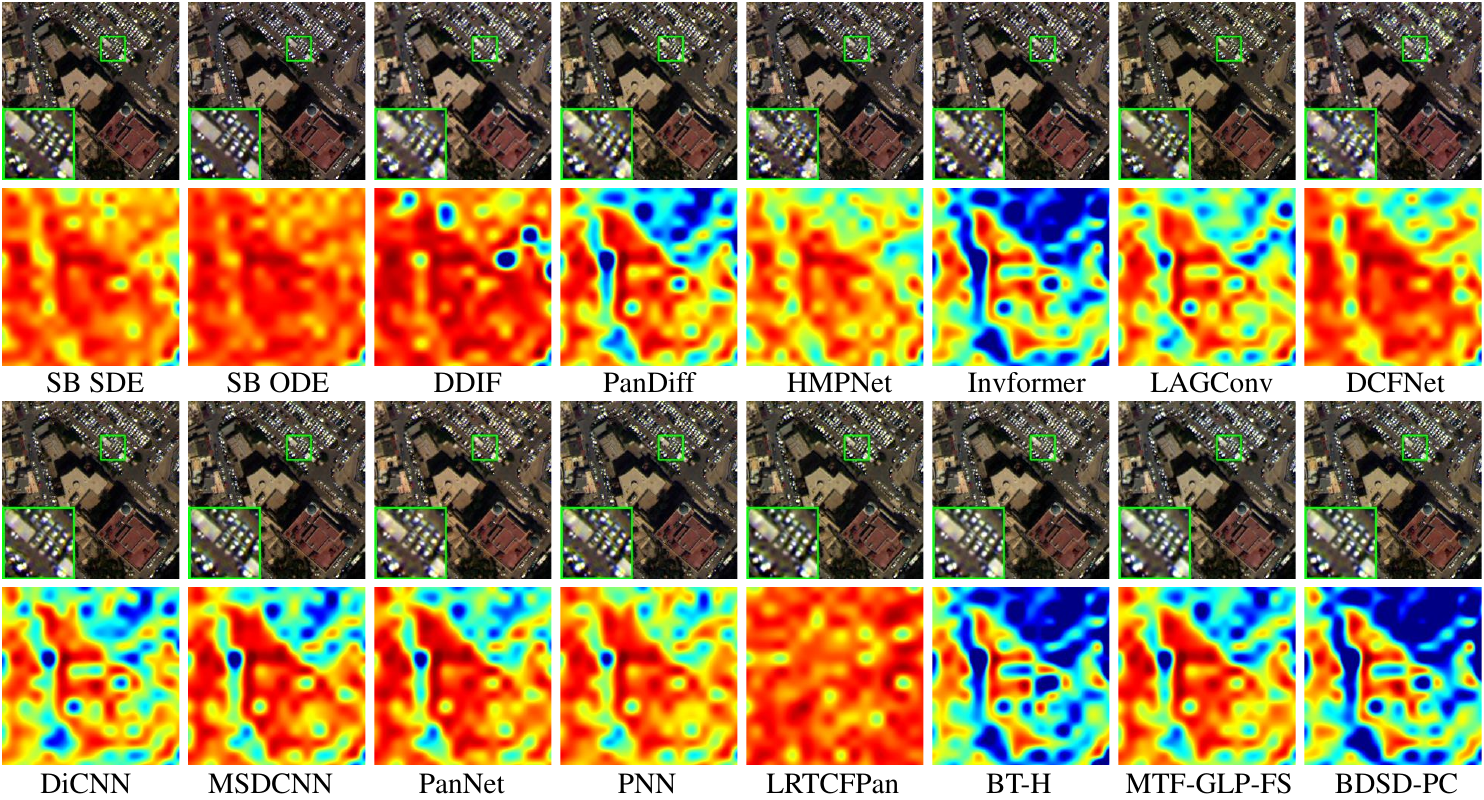}
	\caption{Illustration of fused full-resolution image and corresponding HQNR map on WV3 full-resolution test set. The color red in the HQNR map indicates a value close to 1, while the color blue represents a value close to 0. A higher value in the HQNR map indicates a better pansharpening performance.}
	\label{fig:wv3-full-comp}
\end{figure}
The reduced assessments of WV3, GF2, and GF5-GF1 datasets are provided at the left panel of Tabs.~\ref{tab:pansharpening} and \ref{tab: gf5_gf1_reduced_full}. Some error maps are shown in Fig.~\ref{fig:wv3-reduced-comp}.
We can see that all model-based methods still struggle to beat the first DL-based method PNN. With the continuous design and optimization of network architectures, regressive deep learning-based methods have shown significant improvements in their performance on reduced-resolution data, achieving near state-of-the-art (SOTA) results.

The exploration of PanDiff for pansharpening is commendable. However, due to the choice of the $\epsilon$-learning objective (readers are encouraged to refer to \cite{improved_ddpm} for details about $\epsilon$-learning), it has achieved only a relatively average fusion performance. It is worth noting that PanDiff still utilizes the DDPM ancestral sampling strategy~\cite{ho2020denoising}, which significantly slows down the sampling process. It requires 2000 network forward steps to sample a single fused image.
In contrast, DDIF has proposed a more efficient network architecture and alleviated the issues of DPM-based methods by using 25-step DDIM sampling~\cite{ddim}, achieving state-of-the-art (SOTA) performance.
However, due to the forward strategy of DPM, DDIF starts sampling from pure Gaussian noise, overlooking the inherent characteristics of the pansharpening task. Additionally, the adopted SDE-based sampling curvature in DDIF is curved, leading to larger discretization errors and making efficient sampling more challenging.

The proposed SB SDE outperforms DPM-based methods including PanDiff and DDIF.
Alternatively, SB ODE demonstrates a more efficient sampling procedure (5-steps sampling \textit{v.s.} only 1-step sampling) that achieves competitive pansharpening performance.

\subsection{Full Assessments}
The full assessments are shown in the right panel of Tabs.~\ref{tab:pansharpening} and \ref{tab: gf5_gf1_reduced_full}.
SB SDE exhibits full-resolution metrics on the WV3 dataset that are second only to DDIF. Its comprehensive metric HQNR surpasses all model-based, DL-based regressive, and DPM-based methods on the GF2 dataset. Additionally, it demonstrates highly competitive performance on the GF5-GF1 dataset.
For SB ODE, its highly efficient one-step sampling makes it an excellent pansharpening method compared to other DPM-based methods.
As depicted in Fig.~\ref{fig:wv3-full-comp}, both the proposed SB SDE and ODE demonstrate satisfying fusion results. The fused images exhibit the spectral characteristics of the LRMS and the spatial features of the PAN.

\subsection{Ablation Study}
In this subsection, we ablate the key factors in the SB SDE/ODE learning and sampling and provide experimental results on typical pansharpening WV3 dataset.
\subsubsection{Different Parameterization}
In Sect.~\ref{sect: bmp-train-sampling-algos.}, we provide four different learning objectives for SB SDE/ODE:
\begin{enumerate}
	\item[(a)] bridge endpoint $X_0$;
	\item[(b)] bridge length $Y_1-X_0$;
	\item[(c)] bridge posterior length $Y_t-X_0$;
	\item[(d)] bridge endpoint with score $[X_0, s]$ (SB SDE only).
\end{enumerate}
We design three ablations on the learning objectives of the WV3 dataset and the results are provided in the upper panel of Tab.~\ref{tab:ablation}.

\begin{table}[!t]
	\centering 
	\caption{Result on the GF5-GF1 reduced-resolution and full-resolution datasets. Some conventional methods (the first six rows) and the DL-based approaches are compared. The best results are in \best{red} and the second best results are in \second{blue}.}
	\label{tab: gf5_gf1_reduced_full}
	\setlength\tabcolsep{2.5pt}
	\resizebox{\linewidth}{!}{
		\begin{tabular}{l|ccccc|ccc}
			\toprule
			\hline
			\multirow{2}{*}{Methods}& \multicolumn{5}{c|}{Reduced Resolution (RR): Avg$\pm$std} & \multicolumn{3}{c}{Full Resolution (FR): Avg$\pm$std} \\  & PSNR & SSIM & Q2n & SAM & ERGAS & $D_\lambda$ & $D_s$ & HQNR \\
			\hline
			CNMF~\cite{CNMF} & 44.25$\pm$3.89 & 0.9823$\pm$0.0122 & 0.742$\pm$0.177 & 0.851$\pm$0.213 & 2.761$\pm$0.767 & 0.045$\pm$0.083 & 0.059$\pm$0.050 & 0.898$\pm$0.084 \\
			Hysure~\cite{Hysure} & 42.52$\pm$4.52 & 0.9723$\pm$0.0137 & 0.732$\pm$0.146 & 1.305$\pm$0.406 & 3.677$\pm$1.317 & 0.041$\pm$0.076 & 0.074$\pm$0.105 & 0.887$\pm$0.117 \\
			GSA~\cite{GSA} & 44.99$\pm$5.34 & 0.9795$\pm$0.0119 & 0.754$\pm$0.131 & 1.200$\pm$0.332 & 2.898$\pm$0.956 & 0.053$\pm$0.103 & 0.067$\pm$0.062 & 0.882$\pm$0.101 \\
			LTTR~\cite{LTTR} & 47.15$\pm$2.91 & 0.9897$\pm$0.0028 & 0.844$\pm$0.098 & 2.159$\pm$0.286 & 5.808$\pm$2.732 & 0.099$\pm$0.123 & 0.047$\pm$0.023 & 0.860$\pm$0.106 \\
			LTMR~\cite{LTMR} & 45.52$\pm$2.42 & 0.9898$\pm$0.0030 & 0.849$\pm$0.108 & 1.595$\pm$0.344 & 2.720$\pm$1.277 & 0.057$\pm$0.103 & 0.036$\pm$0.018 & 0.910$\pm$0.105 \\
			MTF-GLP-HS~\cite{mtf-glp-hs} & 45.60$\pm$5.82 & 0.9837$\pm$0.0120 & 0.777$\pm$0.146 & 0.856$\pm$0.265 & 2.848$\pm$1.154 & 0.030$\pm$0.048 & 0.075$\pm$0.105 & 0.897$\pm$0.106 \\
			\hline
			ResTFNet~\cite{ResTFNet} & 46.98$\pm$2.11 & 0.9934$\pm$0.0022 & 0.850$\pm$0.100 & 0.906$\pm$0.130 & 3.323$\pm$3.123 & 0.042$\pm$0.078 & 0.088$\pm$0.059 & 0.874$\pm$0.099 \\
			SSRNet~\cite{SSRNet} & 45.49$\pm$2.69 & 0.9880$\pm$0.0047 & 0.850$\pm$0.094 & 1.039$\pm$0.210 & 4.863$\pm$4.161 & 0.117$\pm$0.140 & 0.054$\pm$0.019 & 0.836$\pm$0.134 \\
			Fusformer~\cite{hu2022fusformer} & 49.74$\pm$4.64 & 0.9914$\pm$0.0031 & 0.891$\pm$0.076 & 0.638$\pm$0.155 & 4.761$\pm$0.592 & {0.030$\pm$0.056} & 0.040$\pm$0.025 & 0.931$\pm$0.064 \\
			HSRNet~\cite{HSRNet} & 49.81$\pm$3.05 & 0.9964$\pm$0.0016 & 0.888$\pm$0.081 & 0.693$\pm$0.139 & 0.901$\pm$0.451 & 0.038$\pm$0.073 & 0.047$\pm$0.020 & 0.917$\pm$0.075 \\
			DHIF~\cite{DHIF} & {55.35$\pm$4.20} & {0.9982$\pm$0.0009} & {0.929$\pm$0.076} & {0.309$\pm$0.062} & {0.885$\pm$0.388} & \best{0.031$\pm$0.057} & \best{0.034$\pm$0.022} & \best{0.937$\pm$0.062} \\
			\hline
			PanDiff~\cite{pandiff} & 50.43$\pm$2.89 & 0.9958$\pm$0.0017 & 0.903$\pm$0.080 & 0.633$\pm$0.108 & 1.877$\pm$1.393 & 0.033$\pm$0.059 & 0.041$\pm$0.027 & 0.928$\pm$0.063 \\
			DDIF~\cite{DDIF} & \second{56.40$\pm$3.82} & \second{0.9984$\pm$0.0007} & \second{0.938$\pm$0.064} & \second{0.273$\pm$0.049} & \second{0.845$\pm$0.505} & \second{0.033$\pm$0.057} & \second{0.035$\pm$0.021} & \second{0.933$\pm$0.057} \\
			Proposed(5-step SDE) & \best{60.94$\pm$2.03} & \best{0.9995$\pm$0.0006} & \best{0.964$\pm$2.031} & \best{0.233$\pm$0.031} & \best{0.679$\pm$0.541} & 0.038$\pm$0.071 & 0.040$\pm$0.020 & 0.924$\pm$0.071\\
			Proposed(1-step ODE) & 50.64$\pm$2.62 & 0.9971$\pm$0.0012 & 0.909$\pm$0.081 & 0.743$\pm$0.066 & 1.721$\pm$1.265 & 0.036$\pm$0.066 & 0.040$\pm$0.021 & 0.926$\pm$0.069
			\\ 
			\hline
			Ideal value & $+\boldsymbol{\infty}$ & \textbf{1} & \textbf{1} & \textbf{0} & \textbf{0} & \textbf{0} & \textbf{0} & \textbf{1} \\
			\hline
			\bottomrule
		\end{tabular}
	}
\end{table}

\begin{table}[!t]
	\centering
	\caption{Ablation studies on learning objectives and training timesteps. The best options are in \best{red}. The left-hand side of ``/'' represents the results of SB SDE, while the right-hand side represents SB ODE's.}
	\label{tab:ablation}
	\setlength\tabcolsep{1pt}
	\resizebox{0.6\linewidth}{!}{%
		\begin{tabular}{c|c|cccc}
			\toprule[0.8pt]
			\hline
			Ablations                   & \multicolumn{1}{c|}{\diagbox[width=6.7em]{Options}{Metrics}} & SAM & ERGAS & Q2n & SCC \\ \hline
			\multirow{4}{*}{Objectives} & $X_0$  &      2.73/3.67        &    1.99/2.74  &     0.921/0.886 &   0.989/0.974  \\
			& $Y_1-X_0$    &   2.82/3.69 & 2.16/2.78   &  0.918/0.880    & 0.987/0.973 \\
			& $Y_t-X_0$   &  2.76/3.67              & 2.04/2.76    &  0.919/0.885         &  0.986/0.974   \\
			& $[X_0, \epsilon]${\footnotesize (SDE only)}      &  2.74 & 2.00    &   0.921    &   0.898    \\ \hline
			\multirow{2}{*}{\makecell{Training\\Timesteps}}  & 1000           &  2.72/3.63    &    1.98/2.70 & 0.920/0.889      &   0.990/0.976   \\
			& 500           &  2.73/3.66     &  1.99/2.72   &  0.921/0.887 &       0.989/0.975 \\
			\hline 
			\bottomrule[0.8pt]
		\end{tabular}%
	}
\end{table}
We discovered that utilizing the score as the learning objective resulted in an unstable training procedure and slow convergence. However, when using the bridge length as the objective, we observed a more stable training loss and decent fusion performance. Although switching to bridge length with score improved the performance, it reintroduced training instability. Nevertheless, our default setting, which employs the bridge endpoint, outperforms other objectives and achieves the best fusion metrics.


\subsubsection{Timesteps of Training and Fast Sampling}
As demonstrated in DDPM~\cite{ho2020denoising} and iDDPM~\cite{improved_ddpm}, the choice of training timestep has a significant impact on the final generation performance. Previous studies have acknowledged that larger timesteps can enhance the model's generation ability. To investigate the influence of the timestep on bridge matching, we conducted training with different timestep values: 1000 steps, 500 steps, and 100 steps. The corresponding metrics are listed in the lower panel of Tab.~\ref{tab:ablation}. 

\begin{figure}
	\centering
	\begin{minipage}[t]{0.5\linewidth}
		\includegraphics[width=\linewidth]{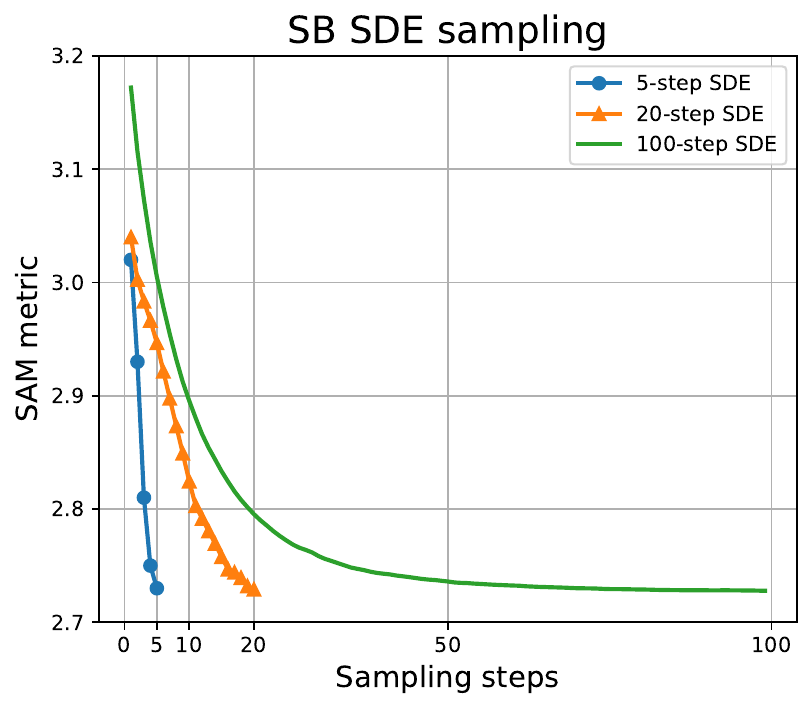}
	\end{minipage}\begin{minipage}[t]{0.5\linewidth}
		\includegraphics[width=\linewidth]{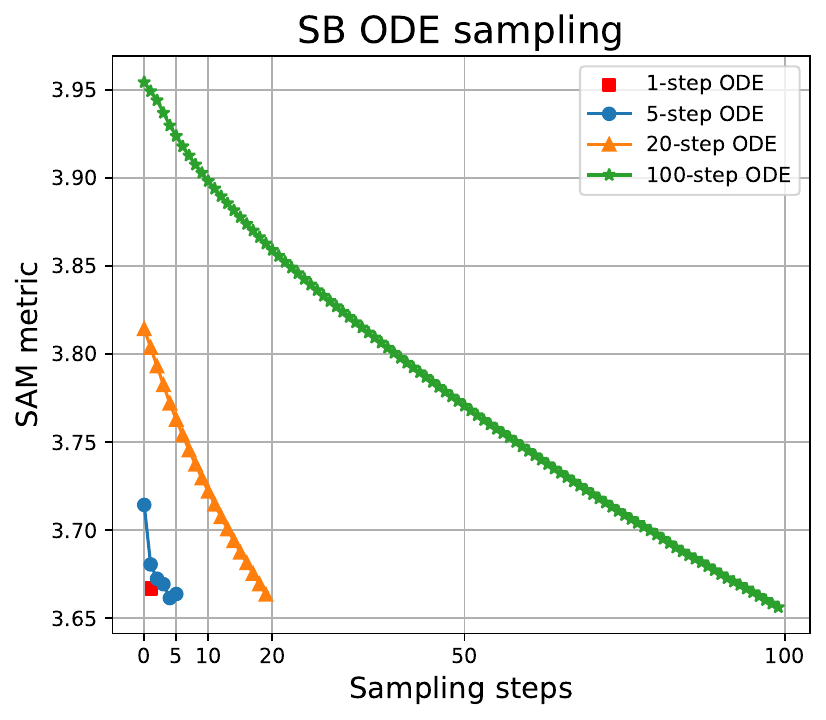}
	\end{minipage}
	\caption{SB SDE/ODE performance of default setting with various sampling steps on WV3 dataset.}
\end{figure}

\subsubsection{Better Architecture for Schr\"odinger Matching}
To validate the suitability of the designed SBM-Net for SB matching training, we employ $X_0$ as the prediction objective, substituting various networks for validation. Additionally, we compare it with other architectures: 1) Vanilla U-Net~\cite{ho2020denoising}, PanDiff~\cite{pandiff}, and DDIF~\cite{DDIF}. The results are shown in Tab.~\ref{tab:arch}. It can be seen that the proposed SBM-Net can outperform all previous DPM-based networks.

\section{Disscussion}
\label{sect:disscuss}

\subsection{Difference and Connections Between SB SDE/ODE, Diffusion, Flow Matching and Rectify Flow}
\noindent \textbf{Diffusion} is an SDE framework bridging the Gaussian distribution and the target distribution. Additionally, \cite{song2019generative} removed the stochastic term from the SDE to construct the PF ODE, which can be expressed in discrete form:
\begin{equation}
	Y_t=\alpha_t Y_0+\sigma_t Z, Z\sim \mathcal N(0, \mathbf I).
\end{equation}
Previous works~\cite{ho2020denoising,song2019generative,song2020score} on diffusion propose various methods for constructing ODE formulations:
\begin{subequations}
	\begin{align}
		&\begin{cases}
			&\text{(sub-)VP ODE}: \alpha_t = \exp\left(-\frac{1}{4} a(1-t)^2-\frac{1}{2}b(1-t)\right),\\
			&\text{VP ODE}: \sigma_t = \sqrt{1-\alpha_t^2}, \text{sub-VP ODE}: \sigma_t = 1 - \alpha_t^2,\label{eq:vp-sde}\\
			&\text{where } a = 19.9, b=0.1;\\
		\end{cases}\\
		&\begin{cases}
			&\text{VE ODE}: \alpha_t=1,\sigma_t=\sigma_{\min}\sqrt{r^{2(1-t)}-1},\\
			&\text{s.t., } \sigma_{\max}:=r\sigma_{\min}, Y_1=Y_0+\beta_1 Z\approx \sigma_{\max}Z
		\end{cases}
	\end{align}
\end{subequations}
All of them can be regarded as interpolations between two distributions. In this view, the cost incurred by a sampler transferring one distribution to another depends on the distance and curvature of its trajectory. When the coupling $(Y_0, Z)$ is given, the distance is prefixed, and the cost depends on the curvature. For different SDEs or ODEs, $\alpha_t$ and $\beta_t$ determine the curvature of the trajectory.
From the above expressions, it is evident that the conventional ODEs previously used do not yield straight trajectories due to the nonlinear relationship between $\alpha_t$ and $\sigma_t$. Consequently, these ODE methods require large sampling steps to achieve satisfactory sampling outcomes. In contrast, our method produces straight sampling trajectories and demonstrates excellent sampling efficiency and minimal truncation errors even under the coarsest discrete method (Euler method).
On the other hand, as pointed out in Sect.~\ref{sec:deg-ode-sde}, the Diffusion framework facilitates transport between Gaussian and target distributions, but it is not conducive to inverse problems with a known prior distribution. Additionally, fixing one end to be a Gaussian distribution is unnecessary. The proposed SB SDE/ODE addresses this issue. Moreover, one can notice that the SB SDE can be degraded to DPM when setting $p_Y$ as Gaussian.

\noindent \textbf{Flow Matching} extends Countinous Normalizing Flow~\cite{chen2018neural} and utilizing the change of varibles of \cite{chen2021likelihood} to establish an ODE generative framework. A vector ﬁeld $v_t$ can be used to construct a time-dependent diffeomorphic map $\phi: [0,1]\times \mathbb R^d\to \mathbb R^d$, which is defined by an ODE:
\begin{subequations}
	\begin{align}
		\frac{d}{dt}\phi_t(y)&=v_t(\phi_t(y)),\\
		\phi_0(y)&=y.
	\end{align}
\end{subequations}

\begin{table}[t]
	\centering
	\caption{Ablation study on different architecture for SB SDE.}
	\label{tab:arch}
	\resizebox{0.5\linewidth}{!}{
		\begin{tabular}{c|cccc}
			\toprule
			\hline
			Arch. & SAM & ERGAS & Q2n & SCC \\
			\hline
			U-Net & 3.24 & 2.43 & 0.907 & 0.982 \\
			PanDiff & 3.10 & 2.28 & 0.912 & 0.985 \\
			DDIF & 2.76 & 2.03 & 0.919 & 0.989 \\
			SBM-Net & 2.73 & 1.99 & 0.921 & 0.989 \\
			\hline
			\bottomrule
	\end{tabular}}
	
\end{table}
The proposed SB ODE can be degraded to the flow matching ODE.
\begin{corollary}
	When $\beta_t$ is a constant $\beta$ over $t$, the velocity $v_t = (Y_t-X_0)/t$ and the mean in Eq.~\eqref{eq: mean and variance} turns into $\mu_t=(1-t)X_0+t Y_1$, which returns to \textbf{Example II} in Flow Matching~\cite{lipman2023flow}.
\end{corollary}
\begin{proof}
	When $\beta_t = \beta$ is a constant, the term $\frac{\beta_t}{\sigma_t^2}$ decays in $\mathcal O(1/t)$ because $\sigma_t^2=\int_0^t \beta_\tau d\tau = \beta t$. Incorporating with Corollary~\ref{cor: SB ODE}, SB ODE velocity $v_t=(Y_t-X_0)/t$ and the posterior mean $\mu_t=(1-t) X_0+tY_1$.
\end{proof}

\noindent \textbf{Rectify Flow} is derived from the perspective of the transport mapping problem, resulting in ordinary differential equation (ODE) formulations consistent with Flow Matching formulations. In order to further rectify the straightness of the flow, the Reflow~\cite{liu2024instaflow} possesses the ability to reduce transport costs, which parallels the proposed SB ODE.

\subsection{Relation between SB SDE with OT Problems}
\label{sect:sb-sde-with-ot-relation}
We first introduce entropic OT formulation, which is based on Kantorovich's OT problem~\cite{villani2009optimal}. We use $H(\pi)$ to express the entropy of a distribution $\pi$, and $D_{KL}(\pi_1, \pi_2)$ to denote the Kullback-Leibler divergence.
One widely-used entropic OT formulation with $H(\pi)$ based on Kantorovich's OT (with quadratic cost) is:
\begin{align}
	\inf_{\pi \in\Pi(p_X, p_Y)}\int_{\mathcal X\times \mathcal Y}\frac{\|x-y\|^2}{2} d\pi (x,y) - \epsilon H(\pi),
	\label{eq:entropic-ot}
\end{align}
where $\pi$ is a transport plan between $p_X$ and $p_Y$, $\mathcal X, \mathcal Y$ are $D$-dimensional Euclidean space, and $\Pi(p_X, p_Y) \subset \mathscr P_2(\mathcal X\times \mathcal Y)$ is the set of probability distribution on $\mathcal X\times \mathcal Y$ with marginals $p_X$ and $p_Y$. The minimizer is unique due to the convexity of $H(\pi)$.

We denote the path measures $\mathbb{Q}$ and $\mathbb{P}$ by the forward~\eqref{pan-sb-f2} and backward SDEs~\eqref{pan-sb-b2}, respectively. Within the SB framework \eqref{eq: sb-framework}, the SB problem gives rise to a stochastic process $\mathbb{P}$ with marginal distributions $p_X$ and $p_Y$ at $t = 0$ and $t = 1$, respectively, which has minimal KL divergence with the prior path measure $\mathbb Q$.

Now, by using Problem 4.2 in \cite{chenstochastic2021}, we denote OT plan $\pi^{\mathbb P}$ as the joint distribution of the stochastic process $\mathbb P$ whose marginal distrbutions are $\pi_0^{\mathbb P}$, $\pi_1^{\mathbb P}$. The coupling $(X, Y)$ at times $t=0,1$ determines the process $\mathbb Q$ and further denoted as the conditional form $\mathbb P_{|X,Y}$. The KL term in Eq.~\eqref{eq: sb-framework} can be decomposed as~\cite{vargas2021solving},
\begin{equation}
	\begin{aligned}
		D_{KL}(\mathbb Q\| \mathbb P)=&D_{KL}(\pi^{\mathbb Q}\| \pi^{\mathbb P})\\
		+&\int_{\mathcal X\times \mathcal Y} D_{KL}(\mathbb Q_{|X,Y}\| \mathbb P_{|X,Y}) d\pi^{\mathbb P}(X,Y),
		\label{eq: OT-KL}
	\end{aligned}
\end{equation}
the first term is the KL divergence at $t=0, 1$ (\textit{i.e.,} SDE start and end points), and the second term denotes the similarity of two processes during intermediate times. Moreover, the first term can be rewritten as,
\begin{align}
	D_{KL}(\pi^{\mathbb Q}\| \pi^{\mathbb P}) =  \int_{\mathcal X\times \mathcal Y} \frac{\|X - Y\|^2}{2\epsilon} d\pi^{\mathbb P}(X, Y) - H(\pi^{\mathbb P}) + const.
	\label{eq:entropic-int}
\end{align}
Proofs are in Appendix~\ref{app:entropy-int}. In Proposition 2.3 of \cite{leonard2013survey}, if $\mathbb P^*$ is the solution to Eq.~\eqref{eq: sb-framework}, then $\mathbb P^*_{|X,Y}=\mathbb Q_{|X,Y}$, hence, $\forall (X, Y)$, one can set the second term in Eq.~\eqref{eq: OT-KL} to zero and optimize Eq.~\eqref{eq: sb-framework} over process $\mathbb P$, which means:
\begin{equation}
	\begin{aligned}
		\inf_{\mathbb P\in \mathcal F(p_X, p_Y)} D_{KL}(\mathbb Q\| \mathbb P)&=\inf_{\mathbb P\in \mathcal F(p_X, p_Y)} D_{KL}(\pi^{\mathbb Q}\| \pi^{\mathbb P}) \\
		&= \inf_{\pi^{\mathbb P}\in \Pi(p_X, p_Y)} D_{KL} (\pi^{\mathbb Q}\| \pi^{\mathbb P}),
		\label{eq:sb-kl-to-ot-rel}
	\end{aligned}
\end{equation}
It implies the solutions of Eqs.~\eqref{eq:entropic-ot} and \eqref{eq:entropic-int} coincide and the SB problem can be simplified to the entrpoic OT problem.

\subsection{Why Stochasticity Can Help SB Matching}
Informally, SB SDE and SB ODE differ only by a stochastic Wiener term in their discrete formulations. This term introduces stochasticity along the trajectory in SB SDE.
From the main results, it is evident that SB SDE yields better outcomes compared to SB ODE. This is attributed to the characteristic of pansharpening tasks requiring the restoration of high-frequency information. The trajectory of SB ODE is deterministic and do not introduce any Gaussian noise. Therefore, to recover the high-frequency information of HRMS, the network needs to directly learn to generate the high frequencies, which is challenging. In contrast, with SB SDE, the introduction of Gaussian noise containing high frequencies during the sampling process makes it easier for the network to restore.

\subsection{Curvature of SB ODE compared with Diffusion VP ODE}
For a generative process $\mathbf Y = \{Y_t\}_0^1$ with initial value
$Y_0=X_0$, we informally provide curvature definition which the generative trajectory from a straight path:
\begin{equation}
	C(\mathbf Y)=\mathbb E{\Big\|}Y_1 - X_0 -\cfrac{\partial}{\partial t} Y_t{\Big\|}_2^2.
\end{equation}
The mean curvature $\mathbb E_{Y\sim p(\mathbf Y)} C(\mathbf Y)$ is first considered and related to the truncation error of the ODE solver. Therefore, the zero curvature represents straight sampling trajectory.
As a generative process is a time reversal of the forward process, the curvature of sampling process is determined by the forward process. In previous DPM works~\cite{song2019generative,ho2020denoising}, the forward process is prefixed, thus the optimal curvature of the sampling process can be derived. Consider the proposed SB ODE~\eqref{eq:ot-ode} for example, the average curvature can be derived,
\begin{subequations}
	\begin{align}
		C(\mathbf Y)_{\text{SB ODE}}&=\mathbb{E}_{\mathbf Y,t}\Big\|Y_1 - X_0 -\nabla_t\frac{\hat \sigma_t^2}{\hat \sigma_t^2 +\sigma_t^2}X_0 - \nabla_t\frac{\sigma_t^2}{\hat \sigma_t^2 +\sigma_t^2} Y_1 \Big\|_2^2 \label{eq:cur-sb-ode-def} \\
		&=\mathbb E_{\mathbf Y, t} \|(1-M_t) (Y_1-X_0)\|_2^2, \label{eq:cur-sb-ode} \\
		&\text{where } M_t=\frac{2\beta_t \sigma_t \hat\sigma_t(\sigma_t+\hat\sigma_t)}{(\sigma_t^2+\hat\sigma_t^2)^2}\notag,
	\end{align}	
\end{subequations}
and the average curvature of PF ODE (\textit{e.g.}, VP SDE) is,
\begin{subequations}
	\begin{align}
		C(\mathbf Y)_{\text{VP ODE}}&=\mathbb{E}_{\mathbf Y,t}\Big\|Y_1-X_0-\nabla_t(\alpha_t Y_0+\beta_t Z)\Big\|_2^2 \\
		=\mathbb{E}_{\mathbf Y,t}&\Big\|(1+N_t)X_0\Big\|_2^2+\mathbb{E}_{t}\|1+K_t\|_2^2\mathbf I_d, \label{eq:cur-vp-ode} \\
		\text{\hspace{-0.6em} where, } N_t=&\frac 1 8 [-2a(1-t)^2-4b(1-t)][a(1-t)+b]\alpha_t,\notag\\ 
		K_t&=-\alpha_t(1-\alpha_t^2)^{-\frac 12}N_t.\notag
	\end{align}
\end{subequations}
Following this definition, we perform average curvature numerical check on WV3 test set, the results are only 0.09 and 2.79 of SB ODE and VP ODE (note that $f_t$ is defined by more widely-used Variance Preserving (VP) parameterization~\cite{ho2020denoising,song2020score} as in Eq.~\eqref{eq:vp-sde}). The curvature of the proposed SB-ODE is much smaller than that of the previous DPM ODE, indicating that we can use fewer sampling steps to sample along the trajectory and suffer less sampling truncation errors. The derivations of the curvatures are provided in Appendix~\ref{app:curvature-of-sdes}.

\section{Conclusion}
This paper introduces an SDE/ODE method based on the Schr\"odinger Bridge (SB) to address the limitations of previous DPM-based approaches, such as being limited to transportation between Gaussian and data distributions, low sampling efficiency, and the complexity associated with training objectives and trajectory simulation in previous IPF and likelihood-based SB methods. Our approach achieves SOTA performance on three datasets for pansharpening, laying the groundwork for the application of SB methods in pansharpening tasks.

\section{Acknowledge}
This research is supported by National Key Research and National Natural Science Foundation of China (12271083), 
Development Program of China (Grant No. 2020YFA0714001), and Natural Science Foundation of Sichuan Province (2022NSFSC0501, 2023NSFSC1341, 2022NSFSC1821).

\bibliographystyle{AIMS}
\bibliography{ref_pvsc}

\clearpage
\appendix
\section{Derivations of Eqs.~\eqref{eq:cur-sb-ode} and \eqref{eq:cur-vp-ode}}
\label{app:curvature-of-sdes}

We first derive the curvature of SB ODE.
Using the definition:
\begin{align}
		C(\bf Y)_{\text{SB ODE}}&=\mathbb{E}_{\bf Y,t}\Big\|Y_1 - X_0 -\nabla_t\frac{\hat \sigma_t^2}{\hat \sigma_t^2 +\sigma_t^2}X_0 - \nabla_t\frac{\sigma_t^2}{\hat \sigma_t^2 +\sigma_t^2} Y_1 \Big\|_2^2, \label{eq: cur-sb-ode-definition}
\end{align}
we can deduce the first derivative:
\begin{subequations}
	\begin{align}
		\nabla_t\frac{\hat \sigma_t^2}{\hat \sigma_t^2 +\sigma_t^2}&=\frac{-2\hat \sigma_t^2\beta_t-\hat\sigma_t^2 \nabla_t(\sigma_t^2+\sigma_t^2)}{(\hat\sigma_t^2+\sigma_t^2)^2},\ \text{where } Y_1\sim p_{Y}, X_0\sim p_X \\
		&= -2\beta_t \frac{\hat\sigma_t^2(\hat\sigma_t^2+\sigma_t^2)+\hat\sigma_t^2(-\hat \sigma_t+\sigma_t)}{(\hat\sigma_t^2+\sigma_t^2)^2} \\
		&= -2\beta_t\frac{\hat\sigma_t\sigma_t(\sigma_t+\hat\sigma_t)}{(\hat\sigma_t^2+\sigma_t^2)^2}=:-M_t.
	\end{align}
\end{subequations}
Similarly, we can write the second derivate:
\begin{subequations}
	\begin{align}
		\nabla_t\frac{\sigma_t^2}{\hat \sigma_t^2 +\sigma_t^2}=2\beta_t\frac{\hat\sigma_t\sigma_t(\sigma_t+\hat\sigma_t)}{(\hat\sigma_t^2+\sigma_t^2)^2}=M_t.
	\end{align}
\end{subequations}
By substituting back in to \eqref{eq: cur-sb-ode-definition}, we can derive Eq.~\eqref{eq:cur-sb-ode}.

The curvature derivations of VP ODE is provided below.
Using the definition:
\begin{subequations}
	\begin{align}
		&C(\bf Y)_{\text{VP ODE}}\\
        &=\mathbb{E}_{\bf Y,t}\Big\|Y_1-X_0-\nabla_t(\alpha_t Y_0+\beta_t Z) \Big\|_2^2, \qquad\text{where } Y_1, Z\sim \mathcal N(0, \bf I_d)\\
		&=\mathbb{E}_{\bf Y,t}\Big\|Y_1-X_0-\nabla_t(\alpha_t) Y_0+\nabla_t(\beta_t) Z) \Big\|_2^2 \\
		&=\mathbb{E}_{\bf Y,t}\Big\|Y_1-X_0-\underbrace{\frac 18 [-2a(1-t)^2-4b(1-t)][a(1-t)+b]\alpha_t}_{\nabla_t(\alpha_t)=: N_t}Y_0  \\ &\qquad \underbrace{-\alpha_t(1-\alpha_t^2)^{-\frac 12}\nabla_t(\alpha_t)}_{\nabla_t(\beta_t):=K_t}Z\Big\|_2^2\\
		&=\mathbb{E}_{\bf Y,t}\Big\|(1+K_t)Y_1-(1+N_t)X_0\Big\|_2^2 \\
		&= \mathbb{E}_{\bf Y,t}\Big\|(1+N_t)X_0\Big\|_2^2+\mathbb{E}_{t}\|1+K_t\|_2^2\bf I_d
	\end{align}
\end{subequations}


\section{Proof of Eq.~\eqref{eq:entropic-int}}
\label{app:entropy-int}
\begin{proof}

Recall the prior process $\mathbb Q \in \mathcal F(p_X, p_Y)$ and let the $\mathbb P$ be a probility measure that define a process to recover the process $\mathbb Q$. As $f_t$ is set to 0 in our proposition,
hence, $\pi_0^{Y|X}$ is a Gaussian distribution $\mathcal N(Y|X, \epsilon \bf I_d)$, where $\epsilon$ is defined by $g_t$ in Eq.~\eqref{pan-sb-f2}. The KL term in Eq.~\eqref{eq:entropic-int} can be writed as,
\begin{subequations}
\begin{align}
	D_{KL}(\pi^{\mathbb Q}\| \pi^{\mathbb P})&=-\int_{\mathcal X\times \mathcal Y} \log \frac{d\pi^{\mathbb Q}(X, Y)}{d[X,Y]}d\pi^{\mathbb P}(X,Y)+\int_{\mathcal X\times \mathcal Y}\log\frac{d\pi^{\mathbb P(X,Y)}}{d[X,Y]}d\pi^{\mathbb P}(X,Y)\\
	&=-\int_{\mathcal X\times \mathcal Y} \log \frac{d\pi^{\mathbb Q}(X, Y)}{d[X,Y]}d\pi^{\mathbb P}(X,Y)-H(\pi^{\mathbb P}), \label{eq: d-kl-ot}
\end{align}
\end{subequations}
where $\frac{d\pi(X,Y)}{d[X,Y]}$ denotes the joint density of distribution $\pi$. The first term can be derive:
\begin{subequations}
	\begin{align}
		&\int_{\mathcal X\times \mathcal Y}\log \frac{d\pi^\mathbb Q(X,Y)}{d[X,Y]}d\pi^\mathbb P(X,Y)\\
		&=-\int_{\mathcal X\times \mathcal Y}\log \frac{d\pi^\mathbb Q(Y|X)}{dY}\frac{d\pi^\mathbb Q(X)}{dX}d\pi^\mathbb P(X,Y)\\
		&=-\int_{\mathcal X\times \mathcal Y}\log \frac{d\pi^\mathbb Q(Y|X)}{dY}d\pi^\mathbb P(X,Y)-\int_\mathcal X\int_\mathcal Y\log  \frac{d\pi^\mathbb Q(X)}{dX}d\pi^\mathbb P(Y|X) d\pi_0^\mathbb P (X)\\
		&=-\int_{\mathcal X\times \mathcal Y}\log \frac{d\pi^\mathbb Q(Y|X)}{dY}d\pi^\mathbb P(X,Y)-\int_\mathcal X\log \frac{d\pi^\mathbb Q(X)}{dX}\left[\int_\mathcal Y 1 d\pi^\mathbb P(Y|X)\right] d\pi_0^\mathbb P (X)\\
		&=-\int_{\mathcal X}\int_{\mathcal Y} \log \frac{d\pi^\mathbb Q(Y|X)}{dY}d\pi^\mathbb P(X,Y)-\int_\mathcal X\log \frac{d\pi^\mathbb Q(X)}{dX}d\pi_0^\mathbb P (X)\\
		&=-\int_{\mathcal X}\int_{\mathcal Y} \log \frac{d\pi^\mathbb Q(Y|X)}{dY}d\pi^\mathbb P(X,Y)-\int_{\mathcal X}\log\frac{d\mathbb \pi^\mathbb P_0(X)}{dX}d\mathbb \pi^\mathbb P_0(X)\\
		&=-\int_{\mathcal X\times \mathcal Y}\log \frac{d\pi^\mathbb Q(Y|X)}{dY}d\pi^\mathbb P(X,Y)+H(\pi_0^{\mathbb P})\\
		&=-\int_{\mathcal X\times \mathcal Y}\log \left((2\pi\epsilon)^{-D/2}\exp\left(\frac{\|X-Y\|^2}{2\epsilon}\right)\right)d\pi^{\mathbb P}(X,Y) + H(\pi_0^{\mathbb P})\\
		&=\frac D 2 \log (2\pi \epsilon) +\int_{\mathcal X\times \mathcal Y}\frac{\|X-Y\|^2}{2\epsilon}d\pi^{\mathbb P}(X,Y) + H(\pi_0^{\mathbb P})
	\end{align}
\end{subequations}
We can substitude back to Eq.~\eqref{eq: d-kl-ot} and further obtains,
\begin{equation}
	D_{KL}(\pi^{\mathbb Q}\| \pi^{\mathbb P})=\int_{\mathcal X\times \mathcal Y} \frac{\|X-Y\|^2}{2\epsilon}d\pi^{\mathbb P}(X,Y)-H(\pi^{\mathbb P})+\underbrace{\frac{D}{2} \log(2\pi \epsilon)+H(\pi^{\mathbb P}_0)}_{const\, \text{in Eq.~\eqref{eq:entropic-int}}}.
\end{equation}

\end{proof}

\section{SB ODE Reduce the Transport Cost}
Based on the relation of SB problem with OT problem, we conclude a corollary which claims the SB ODE can reduce the transport cost.

\begin{corollary}
	Given coupling $(X_0, Y_1)$, for any convex cost function $c: \mathbb R^d\to \mathbb R$, we have,
	\begin{equation}
		\mathbb E[c(Y_1-Y_0)]\leq \mathbb E[c(Y_1-X_0)].
	\end{equation}
\end{corollary}

\begin{proof}
The proof is based on Jensen's inequality.
\begin{subequations}
\begin{align}
	\mathbb E[c(Y_1-Y_0)]&=\mathbb E\left[c\left(\int_0^1 [Y_1-p(\hat Y_0|Y_{t+\Delta t}, t)]dt\right)\right] \label{eq:cor-c1-1}\\
	&\leq \mathbb E\left[\int_0^1 c\left(Y_1-p(\hat Y_0|Y_{t+\Delta t}, t)\right)\right] \label{eq:cor-c1-2}\\
	&=\mathbb E\left[\int_0^1 c\left(
		\mathbb E[(Y_1-\hat Y_0)|Y_t]
	\right)dt
	\right]\label{eq:cor-c1-3}\\
	&\leq \mathbb E \left[
		\int_0^1 \mathbb E[c(Y_1-X_0)|Y_t]dt
	\right]\label{eq:cor-c1-4}\\
	&= \int_0^1 \mathbb E\left[c(Y_1-X_0)|Y_t)dt\right]\label{eq:cor-c1-5}\\
	&= \mathbb E[c(Y_1-X_0)] \label{eq:cor-c1-6}.
\end{align}
\end{subequations}
Eq.~\eqref{eq:cor-c1-1} is based on the integral of the sampling trajectory. Eq.~\eqref{eq:cor-c1-2} is derived by using the convexity of the transport function $c$ and Jensen's inequality. Eq.~\eqref{eq:cor-c1-3} is held based on the definition of $p(\hat Y_0|Y_{t+\Delta t, t})$. The inequality of Eq.~\eqref{eq:cor-c1-4} comes from the same marginal with $p$ and $q$ processes. Then, Eq.~\eqref{eq:cor-c1-5} is held due to $\mathbb E[\mathbb E[Y_1-X_0|Y_t]] = \mathbb E[Y_1-X_0]$.
\end{proof}

\section{Constraint Sampling}
During the sampling process, we need to discretize the time steps for sampling. Despite the nearly straight trajectories of the proposed SB SDE/ODE, small sampling time steps can still introduce sampling errors. When sampling errors accumulate on manifolds not covered during training, it can lead to out-of-distribution (OOD) issues. Recently, several DPM and SB-related studies have discussed how to perform score matching or SB matching in constraint spaces~\cite{lou2023reflected,deng2024reflected,yang2024guidance,liu2024mirror}. They choose to train and sample under constraints such as hypercubes or simplexes, which incur additional costs during training (\textit{e.g.}, involving trajectory reflections or dual space projection). Here, we propose imposing gradient constraints during sampling, utilizing $Y_1$ and $c$ (optional) to solve gradients depending on the setting of the inverse problem:
\begin{align}
	\tilde X_0 = \hat X_0 - \psi \nabla_{\hat Y_0} \left(
		\|A_1(\hat X_0) - Y_1\|_2^2 + \|A_2(\hat X_0) - c\|_2^2
	\right)
	\label{eq:grad-constra}
\end{align}
In the context of pansharpening, $Y_1$ is the LRMS and $c$ is PAN. Thus, we can simulate the degradation process by implementing $A_1$ to be the blur kernel and $A_2$ to be the downsampling operator.
This approach pulls the sampled $\hat Y_0$ back to the correct manifold, reducing accumulated truncation errors during small sampling steps. The gradient sampling scheme is summarized in Algo.~\ref{algo:constraint-eval}.

\section{High-order Sampler for SB SDE and SB ODE}
Due to the larger sampling steps leading to increased truncation errors and accumulated errors, using a first-order Euler method can result in decreased sampling performance. Recent work has explored the use of higher-order SDE/ODE sampling methods to reduce truncation errors. For example, the RK method in DPM-solver, the Huen method in EDM, and the PC method in SGM. These methods can all be applied to the proposed SB SDE, maintaining sampling performance even with fewer sampling steps.
We provide pseudocode for the corresponding methods in SB SDE/ODE using the two-order Heun method in Algo.~\ref{algo:eval-heun}.

	\begin{algorithm}[H]
		\caption{Constraint sampling scheme}
		\label{algo:constraint-eval}
		\KwIn{Sampling timesteps $T$, trained $s_\theta$, degraded samples $p_0$, degraded samples $p_Y$, Conditional samples $p_c$ (optional), $\psi$ is the gradient constraint rate.}
		\KwOut{Sampled data $X_0$.}
		\SetKwFunction{linspace}{linspace} 
		$T_s = \linspace(1,0,T)$		\Comment*[r]{Init timestep sequence.}
		$Y_1\leftarrow p_Y, C\leftarrow p_c$;\\
		\For{$t\leftarrow T_s$}{
			$\hat X_0\leftarrow s_\theta(Y_t, C, t)$\Comment*[r]{Predicted endpoint.}
			\Comment{Gradient constraint \eqref{eq:grad-constra}.}
			$\tilde X_0 \leftarrow \hat X_0 - \psi \nabla_{\hat Y_0} \left(\|A_1(\hat X_0) - Y_1\|_2^2 + \|A_2(\hat X_0) - c\|_2^2\right)$;\\
			$Y_{t} \leftarrow p(Y_{t}|\tilde X_0, Y_{t+\Delta t})$\Comment*[r]{Posterior sampling~\eqref{eq: sample}.}
		}
		$X_0\leftarrow Y_0$.
	\end{algorithm}
		\begin{algorithm}[H]
			\caption{Sampling scheme using Heun method}
			\label{algo:eval-heun}
			\KwIn{Sampling timesteps $T$, trained $s_\theta$, degraded samples $p_0$, degraded samples $p_Y$, Conditional samples $p_c$ (optional).}
			\KwOut{Sampled data $X_0$.}
			\SetKwFunction{linspace}{linspace} 
			$T_s = \linspace(1,0,T)$ \Comment*[r]{Init timestep sequence.}
			$Y_1\leftarrow p_Y, C\leftarrow p_c$; $\Delta t\leftarrow 1/T$;\\
			\For{$t\leftarrow T_s$}{
				$\hat X_0^{(1)}\leftarrow s_\theta(Y_t, C, t)$\Comment*[r]{Predicted endpoint.}
				\eIf{$t\neq 0$}{
					
					$t^\prime \leftarrow t+\Delta t$\Comment*[r]{Select a temporary timestep.}
					$Y_{t^\prime} \leftarrow q(Y_{t^\prime} |\hat X_0^{(1)}, t^\prime)$\Comment*[r]{Move towards $Y_1$.}
					$\hat X_0^{(2)}\leftarrow s_\theta(Y_{t^\prime}, C, t^\prime)$\Comment*[r]{Predicted endpoint $Y_{t^\prime}$.}
					$Y_{t} \leftarrow p(Y_{t}|\frac{\hat X_0^{(1)}+\hat X_0^{(2)}}{2}, Y_{t+\Delta t})$;
				}{
					$Y_{t} \leftarrow p(Y_{t}|\hat X_0^{(1)}, Y_{t+\Delta t})$;
				}
			}
			$X_0\leftarrow Y_0$.
		\end{algorithm}

\section{More Visual Results}
In this section, we provide more visual results of WV3 and GF2 test set as shown in Figs.~\ref{fig:wv3-more-results} and \ref{fig:gf2-more-results}.

\begin{figure}[htbp]
	\begin{subfigure}{0.5\linewidth}
		\centering
		\includegraphics[width=\linewidth]{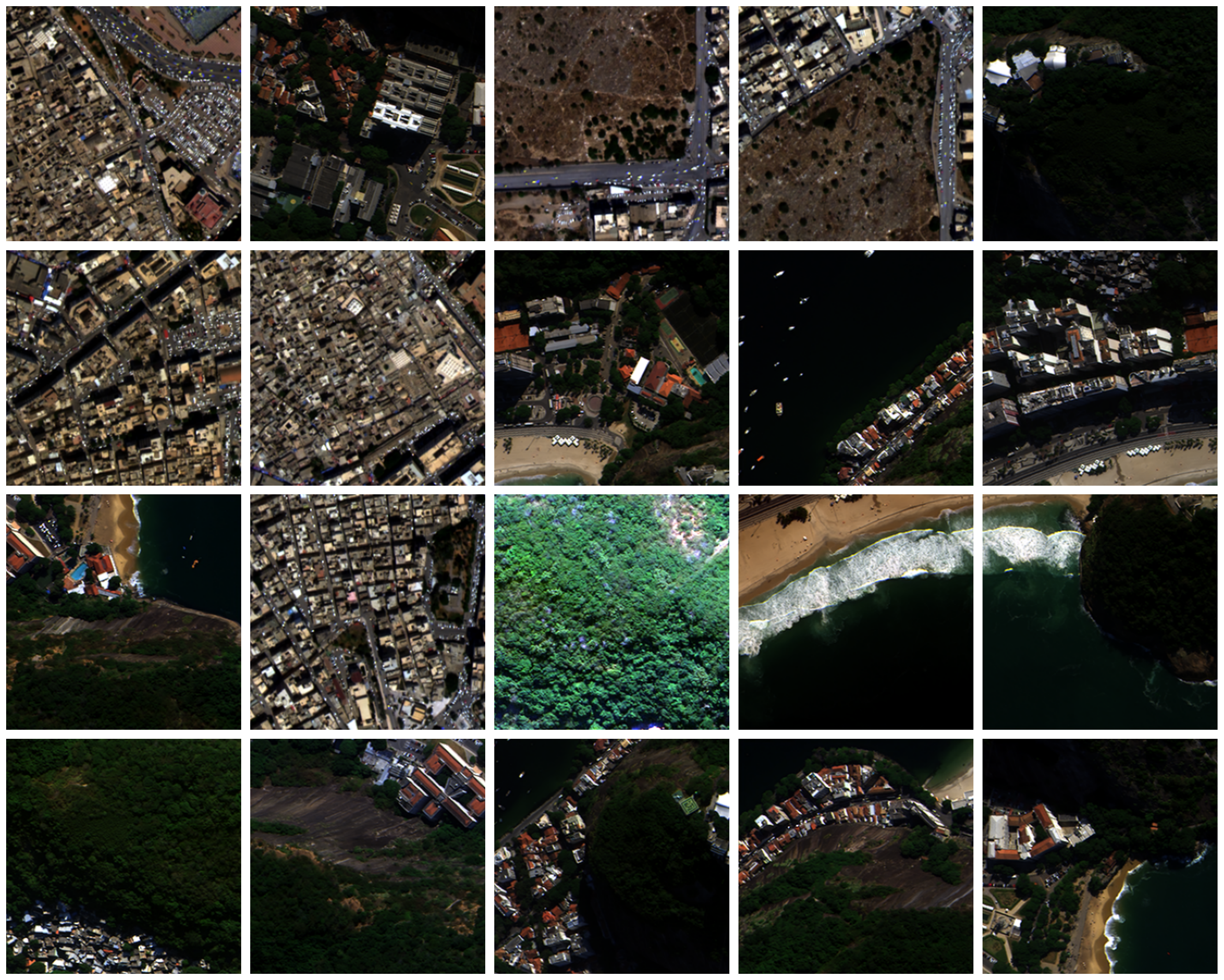}
		\caption{WV3 reduced test set.}
		\label{fig:wv3-more-results}
	\end{subfigure}%
	\begin{subfigure}{0.5\linewidth}
	\centering
	\includegraphics[width=\linewidth]{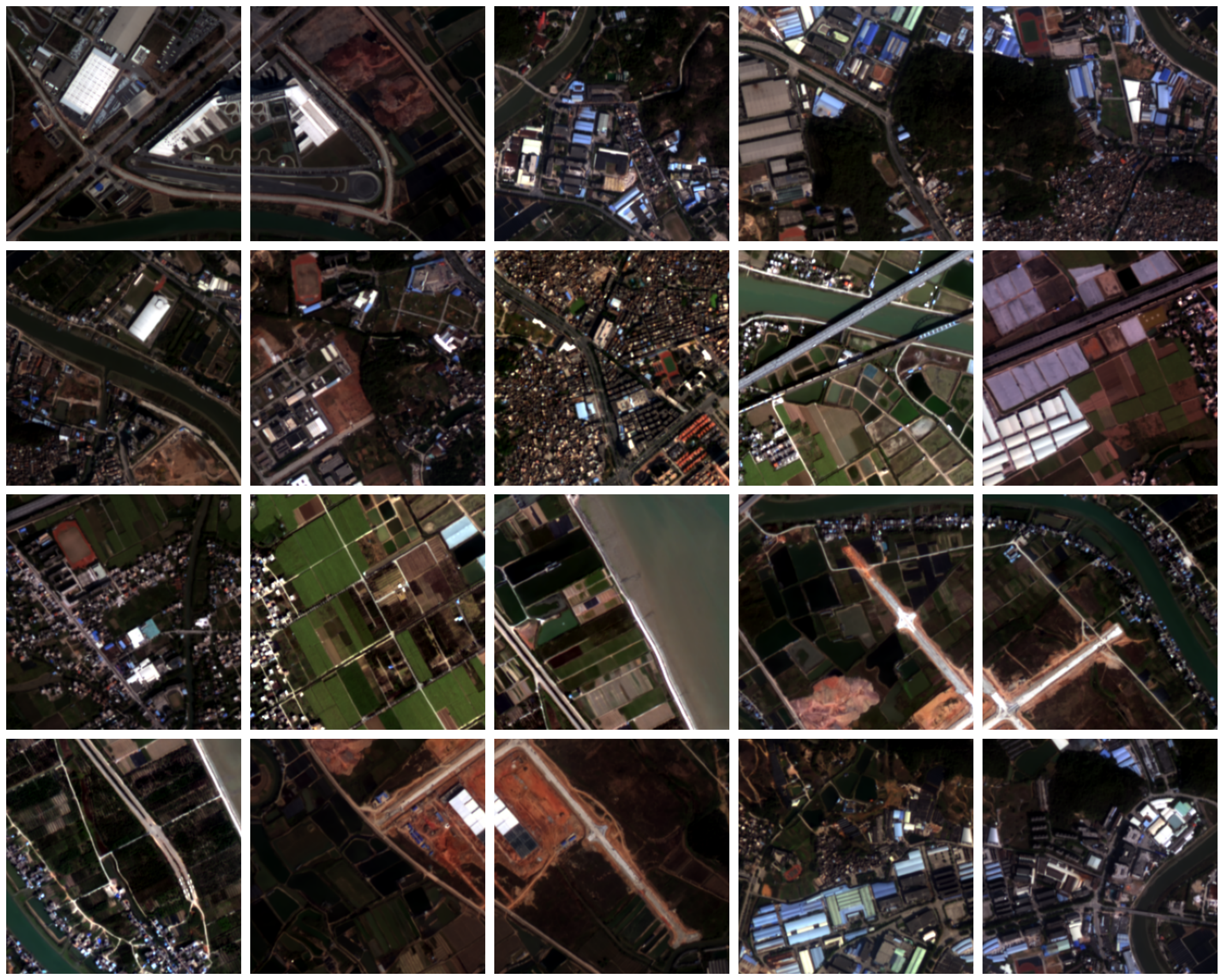}
	\caption{GF2 reduced test set.}
	\label{fig:gf2-more-results}
	\end{subfigure}
	\centering
	\begin{subfigure}{0.5\linewidth}
		\centering
		\includegraphics[width=\linewidth]{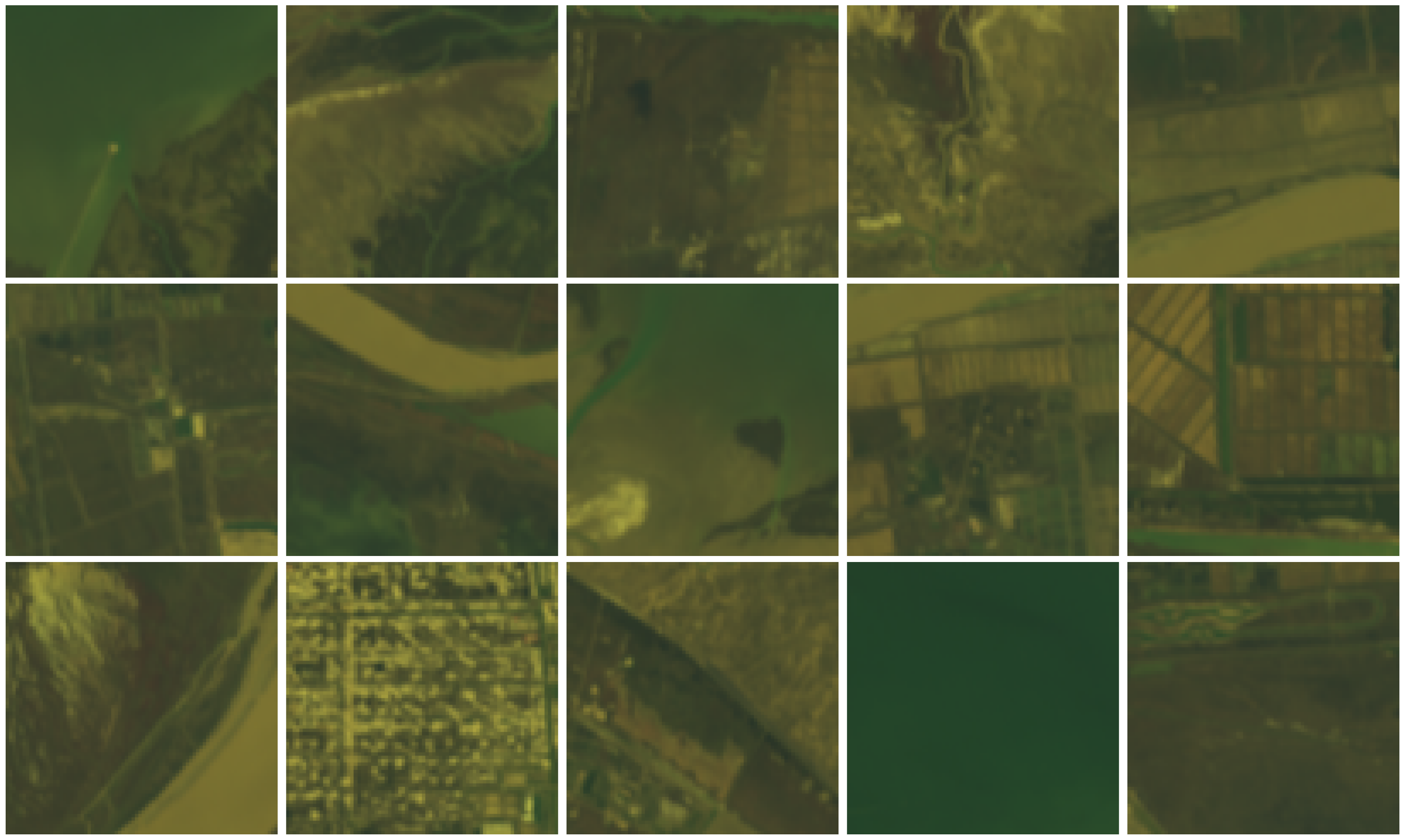}
		\caption{GF5-GF1 reduced test set.}
		\label{fig:gf5-gf1-more-results}
	\end{subfigure}
	\caption{Pansharpening results on WV3, GF2, and GF5-GF1 reduced test sets.}
\end{figure}

\section{Different Parameterizations of $f_t$ and $g_t$}
In Theroem~\ref{theorem:sb-sde}, we set $f_t=0$ and $g_t=\frac{d\sigma_t^2}{dt}=\sqrt{\beta_t}$ as a special parameterization for the SB SDE. In this section, we provide more parameterization for the bridge matching.
Inspired by VP SDE~\cite{ho2020denoising} and sub-VP-SDE~\cite{song2020score}, we can derive (sub-)VP-like SB SDEs. We leave them as the feature work.

\begin{table}[H]
	\centering
	\caption{Different parameterization of SB SDE of $f_t$ and $g_t$.}
	\resizebox{0.8\linewidth}{!}{
		\begin{tabular}{c|ccc}
			\toprule
			Param. & $f_t$ & $g_t^2$ & $q(Y_t|X_0,Y_1)$ \\
			\hline
			VP & $\frac{\log \alpha_t}{dt}Y_t$ & $\frac{d\sigma_t^2}{dt}-2\frac{d\log \alpha_t}{dt}\sigma_t^2$ & $\mathcal N(\alpha_tX_0 +\sigma_tY_1, \sigma_t^2 \bf I_d)$ \\
			sub-VP & $\frac{\log \alpha_t}{dt}Y_t$ & $-2\frac{d\log \alpha_t}{dt}(1+\alpha_t^2)\sigma_t$ & $N(\alpha_tX_0 +\sigma_tY_1, \sigma_t^2 \bf I_d)$ \\
			\hline
			Default & 0 & $\sqrt{\beta_t}$ & $\mathcal N(\frac{\hat\sigma_t^2}{\hat \sigma_t^2+\sigma_t^2} X_0+\frac{\sigma_t^2}{\hat \sigma_t^2+\sigma_t^2} Y_1,\frac{\sigma_t^2\hat\sigma_t^2}{\hat\sigma_t^2+\sigma_t^2}\mathbf I_d)$ \\
			\bottomrule
	\end{tabular}}
\end{table}

\medskip
Received xxxx 20xx; revised xxxx 20xx; early access xxxx 20xx.
\medskip

\end{document}